\documentclass[DIV=calc, paper=letter, fontsize=11pt]{scrartcl}	 

\usepackage[]{units}
\usepackage{lipsum} 
\usepackage[english]{babel} 
\usepackage[protrusion=true,expansion=true]{microtype} 
\usepackage{amsmath,amssymb,amsfonts,amsthm,bm,bbm} 
\usepackage[svgnames, table]{xcolor} 
\usepackage[hang, small,labelfont=bf,up,textfont=it,up]{caption} 
\usepackage{booktabs} 
\usepackage{fix-cm}
\usepackage[]{units}

\usepackage{subfig}

\usepackage{tabularx}

\usepackage[colorlinks=true, linkcolor=blue, citecolor=blue, filecolor=blue,
runcolor=blue, urlcolor = blue]{hyperref} 

\usepackage[capitalize, noabbrev]{cleveref}

\usepackage{fullpage}
\usepackage{sectsty} 

\usepackage{fancyhdr} 
\usepackage{lastpage}

\usepackage{nicefrac}

\usepackage{tikz}
\usetikzlibrary{calc,patterns,decorations.pathmorphing,decorations.markings}
\usepackage{pgfplots}

\usepackage{multicol}
\usepackage{float}

\usepackage[utf8]{inputenc}
\usepackage[english]{babel}
\usepackage{tikz}
\usepackage{comment}

\usetikzlibrary{arrows}

\usepackage{mathtools} 

\usepackage{algorithm, algpseudocode}

\usepackage{enumerate}

\DeclareMathOperator*{\argmin}{arg\,min}

\usepackage[
backend=biber, 
style=numeric-comp,
sorting=none
]{biblatex}

\makeatletter
\newcommand{\multicolinterrupt}[1]{
\setcounter{tempcolnum}{\col@number}
\end{multicols}
#1
\begin{multicols}{\value{tempcolnum}}
}
\makeatother

\definecolor{issuePJA_color}{rgb}{1.0,0.0,0.0}

\definecolor{commentPJA_color}{rgb}{1.0,0.0,0.8}

\definecolor{issueQB_color}{rgb}{1.0,0.8,0.0}

\definecolor{commentQB_color}{rgb}{0.6,0.0,0.8}

\definecolor{rev_color}{rgb}{0.6,0.0,0.0}

\definecolor{atz_table1}{rgb}{0.85, 0.85, 0.85}
\definecolor{atz_table2}{rgb}{0.8, 0.8, 0.8}
\definecolor{atz_table3}{rgb}{0.75, 0.75, 0.75}

\newcommand{\mb}[1]{\mathbf{#1}}
\newcommand{\bsy}[1]{\boldsymbol{#1}}

\newcommand{\R}{\mathbb{R}}
\newcommand{\N}{\mathbb{N}}

\newcommand{\M}{\mathcal{M}}

\newcommand{\norm}[1]{\left\lVert #1 \right\rVert}
\newcommand{\abs}[1]{\left\lvert #1 \right\rvert}
\newcommand{\pn}[1]{\left( #1 \right)}

\newcommand{\set}[1]{\left\{ #1 \right\}}

\newtheorem{theorem}{Theorem}[section]

\usepackage{graphicx}
\DeclareGraphicsExtensions{.png}

\usepackage{lettrine} 
\newcommand{\initial}[1]{ 
\lettrine[lines=3,lhang=0.3,nindent=0em]{
\color{DarkGoldenrod}
{\textsf{#1}}}{}}

\usepackage{caption}

\DeclareOldFontCommand{\bf}{\normalfont\bfseries}{\mathbf}

\usepackage{graphicx}

\usepackage{titling} 

\newcommand{\HorRule}{\color{DarkGoldenrod}

\rule{\linewidth}{1pt}} 

\pretitle{\vspace{-80pt} \begin{flushleft} \HorRule \fontsize{13}{13}
\color{DarkRed} \selectfont} 

\title{Extending Neural Operators: Robust Handling of Functions Beyond the Training Set}

\posttitle{\par\end{flushleft}} 

\preauthor{\vspace{-15pt} \begin{flushleft}\fontsize{12}{12} 
\color{DarkRed}} 
\author{B. Quackenbush$^{1}$, P. J. Atzberger$^{1,2}$ } 
\postauthor{\small \color{Black} 
\\ $[1]$ Department of Mathematics,
University of California Santa Barbara (UCSB). \\
$[2]$ Department of Mathematics, Department of Mechanical Engineering, 
University of California \\ Santa Barbara (UCSB); atzberg@gmail.com;
\url{https://web.atzberger.org/}
\vskip 1.5em
\vspace{-0.6cm}
\HorRule
\end{flushleft} 
} 
\date{\today}

\pgfplotsset{compat = 1.16}

\DeclareGraphicsExtensions{.png}

\begin{document}
\date{}
\maketitle 

\thispagestyle{fancy} 

\vspace{-1.75cm}
\initial{W}\textbf{e develop a rigorous framework for extending neural
operators to handle out-of-distribution input functions. We leverage kernel
approximation techniques and provide theory for characterizing the input-output
function spaces in terms of Reproducing Kernel Hilbert Spaces (RKHSs).  We
provide theorems on the requirements for reliable extensions and their
predicted approximation accuracy. We also establish formal relationships
between specific kernel choices and their corresponding Sobolev Native Spaces.
This connection further allows the extended neural operators to reliably
capture not only function values but also their derivatives. Our methods are
empirically validated through the solution of elliptic partial differential
equations (PDEs) involving operators on manifolds having point-cloud
representations and handling geometric contributions.  We report results on key
factors impacting the accuracy and computational performance of the extension
approaches. 
}

\setlength{\parindent}{5ex}

\section*{Introduction}

Neural operators provide a class of machine learning methods for data-driven
learning of mappings between spaces of
functions~\cite{Bronstein2021,Izenman2012,Stuart2010,
Atzberger2022,Atzberger2020,Chen1995,Kovachki2021a}. Examples include learning
solution operators for Partial Differential Equations (PDEs) and Integral
Operators~\cite{Strauss2007,Pozrikidis1992,Audouze2009,
AtzbergerQuackenbushGNP2024,Bhattacharya2021}, estimators for inverse
problems~\cite{Stuart2010,Kaipio2006,Gelman1995}, estimating geometric
quantities~\cite{AtzbergerQuackenbushGNP2024,Atzberger2020,Quackenbush2025},
and data assimilation~\cite{Stuart2010,Asch2016}.  Neural operators can be used
to provide approaches for non-linear approximations, for learning
representations for analytically unknown operations through training, or
provide accelerations of frequent operations performed in iterative
algorithms~\cite{Chui2018,Izenman2012,Meilua2023,
Atzberger2020,Atzberger2022,Bronstein2021}.

Neural operators use deep neural networks for learning operators by leveraging
a lifting operation to featurize an input function and a collection of 
linear integral operations coupled with non-linear activations. 
This allows for learning mappings that generalize the discrete finite 
layers of multi-layer perceptrons and by integration are less sensitive 
to noise in the input functions and the resolution at which they are 
sampled~\cite{Kovachki2023a}. 

We perform analysis of neural operators to characterize in more detail
the natural function spaces that arise when making different choices 
for the training data. We provide ways to extend neural operator 
mappings beyond the functions in the training set by leveraging 
results in kernel approximation. For judicious choices of the 
training sets this results in input and output
functions spaces corresponding to Reproducing Kernel Hilbert Spaces
that can be viewed as embedded within Sobolev Spaces.  This allows 
for developing extensions that not only ensure convergence of
the functions but also of their derivatives.
We provide theory and theorems that further characterize these spaces and the
key factors impacting the accuracy of the extended neural operators.  

Some related work has been done on architectural design of 
neural operators and on the development of theory for their approximation properties
in \cite{Lu2022,Le2025,Kovachki2023a,Korolev2022,
Pordanesh2025,Mezidi2025,Berner2025,Mezidi2025}.  In \cite{Lu2022,Le2025,Kovachki2023a},
the approximation capabilities are characterized for classes of operators approximable 
by neural operator variants. A proof is given of 
dimension-independent approximation bounds for two-layer networks 
that provide mappings between Banach spaces in \cite{Korolev2022}.
In \cite{Le2025}, analysis is performed for the stability and 
generalization error and neural operators. A few different architectures for 
neural operators have also been proposed with structures motivated by 
results in functional analysis in \cite{Pordanesh2025,Mezidi2025,Berner2025}. 
This includes the development of multi-scale and spectral 
kernels~\cite{Kovachki2023a}, unified frameworks for lifting 
pointwise layers to operator spaces~\cite{Berner2025}, and the introduction of 
nonlinear skip connections derived from functional optimization problems to 
facilitate the training of deeper networks~\cite{Mezidi2025}.

In our work, we consider alternative approaches and generalize the action
of the operators through kernel-based reconstructions instead of relying 
primarily on data-driven interpolation. Most current methods 
depend strongly on the distribution of the training data and interpolation
to determine out-of-distribution performance. In contrast, we construct 
an operator extension method based on the native spaces of a collection of 
training kernels. In previous work, approximation bounds are developed for 
neural operators~\cite{Kovachki2023a,Lu2022} based on general properties.
In our work, we provide a more constructive framework that provides more
specialized error bounds based on the kernel choice and geometry of the domain. 
We also address the setting of operators on embedded manifolds, 
characterizing smoothness penalties that arise when restricting 
ambient kernels to the embedded geometric domain.  This provides more 
detailed bounds with more problem-specific attributes to control sources
of error or to inform architectural and kernel choices. 

To complement our theoretical results, we also perform empirical studies 
showing how the methods perform in practice. This includes investigating
learned Laplace-Beltrami operators and how they perform across manifold shapes 
of varying geometric complexity.  For further improvements in computational 
efficiency, we also introduce an architecture referred to as Separable Geometric Neural 
Operators (SB-GNPs).  We show how our SB-GNPs can be used to perform training which
ensures convergence not only of the functions but also their derivatives. This allows
for Sobolev training that minimizes the $\mathcal{H}^1$ error while maintaining efficiency
on large point clouds. We further evaluate our approaches based on approximate 
Green's Functions and compare the cases of using (i) Gaussian, (ii) Mat\'ern, and 
(iii) Wendland kernels. We show that while Gaussian kernels 
offer high regularity, they induce severe ill-conditioning and poor performance in 
capturing derivative information. In contrast, we show that 
Mat\'ern and Wendland kernels provide stable, accurate extensions that control the 
trade-off between operator approximation and kernel interpolation errors. These provide
for ways to obtain more robust methods for extending the neural operators and 
capturing derivative information for input functions that were not seen in
the training data. We further characterize the role of hyper-parameters of these
neural operator methods through a collection of benchmark studies.

We organize the paper as follows. We discuss background on neural operators,
kernel methods, and our extension approaches in
Section~\ref{sec:operator_extension}. We then discuss theory on our
operator extensions and the 
related Reproducing Kernel Hilbert Spaces (RKHSs) in Section
~\ref{sec:kernel_interpolation}. We provide proofs of our main theorems in
Section~\ref{sec:operator_extension_proof_sec}.  We discuss how our
approaches can be incorporated in practice into neural operator
architectures and ways the training performance can be made
more efficient in Section~\ref{sec:gnp}.  
We demonstrate the extension methods in approximating solution operators
for geometric PDEs and perform 
empirical studies to further characterize their accuracy and other properties 
in Section~\ref{sec:demonstration}.  Our results show 
a few practical ways neural operators can be extended 
to handle robustly input functions beyond the training data.

\vspace{1cm}

\section{Neural Operators}
\label{sec:operator_extension}
Neural operators are used to learn mappings between function spaces  
using as training data evaluations from 
pairs of functions $\set{(f_i(\cdot),
u_i(\cdot))}_{i=1}^N$. For a target
operator $u_i =
\mathcal{G}[f_i]$, a neural operator learns an approximate operator
$\mathcal{G}_\theta$, 
\cite{Kovachki2021a}.  
Neural operators are motivated by 
applications that include speeding up operations
during solving inverse problems~\cite{Kovachki2023a},
accelerating simulations~\cite{azizzadenesheli2024neural}, approximating
solution maps $\mathcal{S}_\theta$ 
for parameteric PDEs~\cite{Kovachki2023a,FNO_2020},
and geometric tasks involving solving PDEs on manifolds
and curvature-driven shape 
deformations~\cite{Quackenbush2025,AtzbergerQuackenbushGNP2024}.

We consider neural operators $w=\mathcal{G}_\theta[u]$ of the form 
\begin{equation} 
    \mathcal{G}_\theta = \mathcal{Q}
\circ \sigma_T\pn{W_T + \mathcal{K}_T} \circ \dots \circ \sigma_0\pn{W_0
+ \mathcal{K}_0} \circ \mathcal{P}.  
\end{equation}
This consists of $T$ layers with the following three
learnable components, (i) a lifting operator $\mathcal{P}[a]$ with $a
\in \R^{d_a}$ where for $a = (\{u(\tilde{x}_i)\}_{i=1}^N,\{x_i\}_{i=1}^N)$ 
a set of features are obtained from the function evaluations to yield $v_0 \in
\R^{d_v}$ with $d_v \geq d_a$, (ii) a composition of layers
involving integral operators $\mathcal{K}[v]$ and 
linear operators $\mathcal{W}[v](x)$ that are processed 
through a non-linear activation $\sigma_i[\cdot]$ to obtain
$v_{i+1} = \sigma_i \pn{Wv_i + \mathcal{K}[v_i] }$, and (iii) a
projection operator $\mathcal{Q}$ that gives a $\R^{d_u}$-valued output function
$w$. The activation of the last layer $\sigma_T$ is typically taken to be the identity.
The trainable sub-components also include within the operator layers the
kernel $\kappa$ and the local function operator $W$.  We collect all trainable
parameters into $\theta$ for the neural operator $\mathcal{G}_\theta$.

For $\mathcal{K}$ we consider integral operators of the form
\begin{equation}
\label{kernel_integral}
\mathcal{K}[v](x) = \int_D \kappa(x, y) v(y) \ d\mu(y).
\end{equation}
The $\mu$ is a measure on $D \subset \R^{d_{v}}$, $v:D \to \R^{d_{v}}$ is the
input function, and $\kappa$ is the operator 
kernel $\kappa(x, y) \in \R^{d_{v}} \times \R^{d_{v}}$.
In practice, the integral is often approximated on a truncated domain $B_r(x)$ 
using 
\begin{equation}
    \label{equ_kernel_approx}
    \tilde{\mathcal{K}}[v_t]( x_j) = 
    \frac{1}{N} \sum_{x_k \in B_r( x_j)} \kappa(x_j, x_k) v_t(x_k).
\end{equation}
This discretization can be interpreted as performing message passing on a
graph~\cite{gilmer2020message,gori2005new}.  
As further motivation for this work and to illustrate the main ideas 
for the extensions, we consider the problem of learning operators 
$u = \mathcal{S}_\theta[f]$ to approximate the solutions to 
a class of PDEs. 

\subsection{Approximation of Solution Operators for Partial Differential Equations.}
Consider the solution operator $u = \mathcal{S}[f]$ for PDEs of the form
\begin{equation}
\label{eq:linear_pde}
\mathcal{L}u(x) = f(x), \quad x \in \Omega.
\end{equation}
Here, $\Omega$ is a bounded domain and 
$\mathcal{L}$ is a bounded linear operator with 
inverse $\mathcal{S} = \mathcal{L}^{-1}$. We consider the case of boundary
conditions that are Dirichlet with $u(x) = 0$ for $x \in \partial \Omega$
or Neumann $\nabla u(x) \cdot n(x) = 0$ with in the latter solutions 
determined uniquely by also requiring $\int u(x) dx = 0$.  
In the case of kernels 
$k_\sigma: \Omega \times \Omega \to \R$ that are symmetric, positive-definite,
and translation invariant, we can express the kernels as $k_\sigma(x, y) =
\Phi(\sigma(x - y))$. We refer to $\Phi(r)$ as the radial function associated
with the kernel $k$. 
We take the shape factor $\sigma > 0$ which controls the
length-scale $\ell_k = \sigma^{-1}$ over which the kernel varies. The
length-scale $\ell_k$ is also referred to as the band-width of the kernel.
This controls the range of decay of the kernel and in 
discretizations the band-width of associated matrices. 

To train the neural operators $\mathcal{S}_\theta$ to approximate solution
operators $\mathcal{S}$, we use solutions of the following instances of the
problem
\begin{equation}
    \label{eq:linear_pde_kernel}
    \mathcal{L}_xg_\sigma(x, y) = k_\sigma(x, y), \quad x, y \in \Omega.
\end{equation}
The target solution operator can be expressed as $\mathcal{S}: k_\sigma \rightarrow g_\sigma$.
We use training data \\
$\set{(k_\sigma(\cdot, x_i), g_\sigma(\cdot,
x_i))}_{i=1}^N$, where $k_\sigma$ serves the role of the input function and 
$g_\sigma$ the target output solution function.

In the case $k_\sigma(x, y)$ has compact support
and integrates to one, we have in the limit of $\sigma \to \infty$ that 
the kernel behaves similar to a Dirac $\delta$-function with
$k_\sigma(x, y) \rightarrow \delta(x- y)$. In this case the $g_\sigma$ 
approximates the Green's Function $G(x, y)$ for elliptic PDEs of
the form of equation~\ref{eq:linear_pde}. Motivated by this, we call 
$\tilde g_\sigma (x,y) = \mathcal{S}_\theta[k_\sigma(x, y)]$ the 
\textit{Pseudo-Green's Functions} of $\mathcal{L}$ with respect to $k_\sigma$.

Since solutions to such PDEs can be constructed through super-position 
using the Green's Function, this 
motivates the relationship to functions 
$f \in \mathcal{H}_k(\Omega)$ in the Reproducing Kernel Hilbert Space (RKHS) 
of $k_\sigma$. More generally, we consider
functions $f$ in the associated Native Space for the 
RKHS \cite{Aronszajn1950,berlinet2011reproducing}.
For such functions, we can use kernel techniques to obtain 
approximations of the function $f$ of the form
\begin{equation}
    \label{eq:rkhs_approximant}
    \tilde f(x) = \sum_{i=1}^{\abs{X}} \alpha_{i} k_\sigma(x, x_i), \quad x_i \in X.
\end{equation}
The $X \subset \Omega$ is a finite discrete set of points that are used 
as centers for the kernel evaluations. 

For a function $u$ that is a solution to the PDE in equation~\ref{eq:linear_pde} 
having as the right-hand-side function $f$, we can use $\tilde{f}$ 
and the Pseudo-Green's Functions to 
obtain the approximate solution 
\begin{equation}
 \label{eq:pseudo_green_solution}
\tilde u(x) = \sum_{i=1}^{\lvert X \rvert} \alpha_i \tilde g_\sigma(x, x_i).
\end{equation}
This procedure illustrates a way to obtain an approximate solution mapping
$\tilde{\mathcal{S}}: \tilde{f} \rightarrow \tilde{u}$ for the operator
$\mathcal{S}$. 

As we discuss in more detail below, this allows us to handle a broad class of
input functions beyond the training data. The explicit use of kernel
approximations allows us more control over the learned operator's behaviors
instead of just relying on the empirical neural network training of
interpolation or extrapolation of responses.  As we also discuss below, we
further can establish theory for the accuracy of our operator extensions and
characterize other key factors that contribute to the errors and robustness of
the approximations obtained from our training protocols and extension methods.

\subsection{Theory for Extending Neural Operators}
\begin{figure}[t]
\centering
\includegraphics[width=0.99\linewidth]{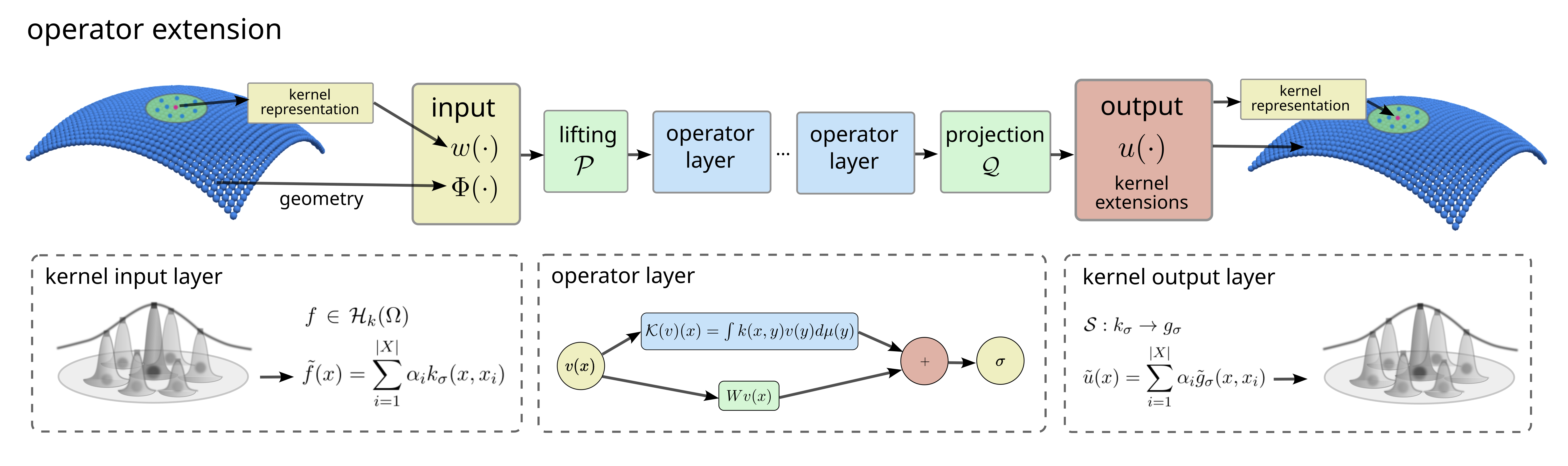}
\caption{\textbf{Operator Extension Methods.} We develop methods for extending
conventional neural operators and geometric neural operators to robustly handle
inputs beyond the functions in the training data. We leverage properties of
kernel approximations.  Neural operators process an input
function $w(\cdot)$ and geometric neural operators also process 
information from the geometry $\Phi(\cdot)$. The operators use
a combination of a featurizing lifting operator $\mathcal{P}$, kernel operator
layers, and a projection operator $\mathcal{Q}$ to obtain the output function
$u(\cdot)$, \textit{(top)}. We obtain representations for input functions in
terms of kernels $k_\sigma$, \textit{(lower-left)}. The kernel
constructions are shown for a subset of the points in the sum.  The operator layer
involves a combination of a kernel integration $\mathcal{K}[v]$ and a linear
operation ${W}[v]$ that is passed through a non-linear activation function
$\sigma(\cdot)$, \textit{(lower-center)}. The kernel representations are then
used to construct the output function $\tilde{u}(\cdot)$, \textit{(lower-right)}.  }
\label{fig:extension_methods}
\end{figure}
For extending neural operators more generally requires theory to characterize 
the roles played by the choice of kernel and other factors impacting
the accuracy and other properties of the operator. To understand how 
an extended neural operator will handle
inputs beyond the training data also requires 
determining the input and output function spaces 
generated by the extensions. 

For analysis of the neural operators, 
we consider a compact set $\Omega \subset \R^d$ a symmetric
positive-definite kernel $k: \Omega \times \Omega \to \R$ whose native space
$\mathcal{H}_k(\Omega)$ is equivalent in norm to the Sobolev Space 
$\mathcal{H}^s(\Omega)$, where
$\lceil s \rceil > d/2 + 2$, 
\cite{brezis2011functional,leoni2017first}. 
In this setting, for any $f \in \mathcal{H}^s(\Omega)$ and $\epsilon > 0$,
there exists a finite set $X \coloneqq \set{x_i}_{i=1}^{N} \subset
\Omega$ and a sequence 
$\bsy{\alpha}$ 
with $\alpha_i =
[\bsy{\alpha}]_i$ such that for any $\epsilon > 0$ 
    \begin{equation}
    \tilde f (x) = \sum_{i=1}^N \alpha_i k_\sigma(x, x_i)
    \end{equation}
    and satisfies
    \begin{equation}
        \label{eq:operator_extension_interpolation_pre_theorem}
        \norm{\tilde f - f}_{\mathcal{H}^s(\Omega)} < \epsilon.
    \end{equation} 
In this case, the extended neural operator 
$\mathcal{S}_{\theta,k}^e$ is given by
\begin{eqnarray}
\label{eq:neural_op_ext}
\tilde{u} = \mathcal{S}_{\theta,k}^e[f] := \sum_{i=1}^{|X|} 
\alpha_i \mathcal{S}_{\theta,k}[k_\sigma(\cdot,x_i)].
\end{eqnarray}
This uses the learned operator $\mathcal{S}_{\theta,k}$
which was trained only on input functions of the form
$k_{\sigma,j} = k_\sigma(\cdot,x_j)$. The steps
of the extension method are illustrated in 
Figure~\ref{fig:extension_methods}.

We establish two theorems on the behaviors and accuracy of the neural operators and our
extensions methods, with more details given in
Section~\ref{sec:operator_extension_proof_sec}.  The first theorem gives a
result for neural operator extensions for functions having domain in $\Omega \subset \R^d$.
The second theorem gives a result for extended geometric neural operators 
when the functions are defined on a manifold $\Omega = \mathcal{M}$.

We remark that our approaches can be applied most immediately for
linear solution operators $\mathcal{S}$. The results also indicate approaches
that could further be generalized to other non-linear classes of operators by
using kernel representations as part of the lifting/projection operators 
$\mathcal{P},\mathcal{Q}$ or using alternatives to 
equation~\ref{eq:neural_op_ext} using approaches from functional 
calculus~\cite{haase2018lectures,haase2006functional,reed2012methods}.
This also includes using our extension techniques at the single-layer level 
internally within existing architectures to extend handling of inputs encountered 
within the internal layer operations. 

In the case of neural operators extended for functions having 
domain in $\Omega \subset \R^d$, we have the following result.
\begin{theorem} \label{thm:op_extend}
Let $\Omega \subset \R^d$ be a compact set and 
$k(\cdot,\cdot)$ a symmetric positive-definite kernel with
a native space that is norm-equivalent to the Sobolev Space  
$\mathcal{H}^s(\Omega)$ for $\lceil s \rceil > d/2 + 2$. 
Suppose the operator $\mathcal{S}$ is bounded with 
$\norm{S(f)}_{\mathcal{H}^1(\Omega)} \leq C_1\norm{f}_{\mathcal{H}^1(\Omega)}$.
For $\delta > 0$ and $\epsilon > 0$,
suppose the approximating operator $\mathcal{S}_{k, \theta}$ satisfies the 
uniform bound
\begin{equation}
\label{equ_thm_S_approx}
\norm{\mathcal{S}_{k, \theta}[k(\cdot, x)] - \mathcal{S}[k(\cdot, x)]}_{\mathcal{H}^1(\Omega)} < \delta,
\quad x \in \Omega,
\end{equation}
and the kernel approximation function $\tilde{f}$ satisfies
\begin{equation}
\label{eq:operator_extension_interpolation}
\norm{\tilde f - f}_{\mathcal{H}^1(\Omega)} < \epsilon.
\end{equation} 
Then, the approximated solution using $\mathcal{S}_{k, \theta}$ and $\tilde f$ 
approximates the true solution to equation~\ref{eq:linear_pde} with the bound    
\begin{equation}
\label{eq:thm_ext_acc}
\norm{\mathcal{S}_{k, \theta}^e[\tilde f] - \mathcal{S}[f]}_{\mathcal{H}^1(\Omega)}
\leq C_1 \epsilon + C_2 \delta,
\end{equation}
    where $C_2 = \norm{\bsy{\alpha}}_{\ell_1}$.
\end{theorem}

We remark that this gives a way to estimate the anticipated accuracy of the
extended neural operator $\mathcal{S}^e$ given a level of sampling of the
domain $X$ and the training accuracy $\delta$.  The $C_1$ depends on the target
operator $S$ and $C_2$ depends on the $\ell_1$-norm of $\bsy{\alpha}$ of the
approximating functions $\tilde{f}$.  For a set of target functions
$\mathcal{F} \subset \mathcal{H}^s$, the bound can be used to provide the level
of training accuracy $\delta$ needed to ensure accuracy of the extension.  The
target functions would in practice be application specific where there is often
some prior knowledge of characteristic physical scales and other properties.
As with any approximation scheme, the more rich the underlying functions the
more resolution or analogous information would be required to ensure accurate
results. The bounds can be used as part of getting a handle on these
requirements for the neural operators used.   

We further develop our results for extending operators on manifolds
$\mathcal{M}$. We view these as $m$-dimensional embedded sub-manifolds
$\mathcal{M} \subset \R^d$ where $m < d$.  We use kernels $k_{\mathcal{M}}$
obtained by restricting a kernel $k$ from $\R^d$ to $\mathcal{M}$.  This gives
$k_{\mathcal{M}} : \mathcal{M} \times \mathcal{M} \rightarrow \R$ where we set
$k_{\mathcal{M}}(x_1,x_2) := k(x_1,x_2)$ with the restriction $x_1,x_2 \in
\mathcal{M}$. These restricted kernels do not necessarily inherit the
properties of $k$.  For example, they may no longer be symmetric, translation
invariant, or isotropic, see Figure~\ref{fig:restricted_kernel}.  This can impact
the approximation properties of the kernels.  

For operator extensions in the manifold setting $\Omega = \mathcal{M}$ using restricted kernels
$k_{\mathcal{M}}$, we have the following result. 
\begin{theorem}
\label{thm:op_extend_manifold}
Let $\mathcal{M} \subset \R^d$ be a $m$-dimensional sub-manifold, and let 
$k: \R^d \times \R^d \to \R$ be a symmetric positive definite kernel. 
Consider the kernel $k_{\mathcal{M}} := k |_{\mathcal{M}}$ obtained by
restricting $k$ to the manifold $\mathcal{M}$.
For $\delta > 0$ and $\epsilon > 0$,
suppose the approximating operator $\mathcal{S}_{k_{\mathcal{M}}, \theta}$ satisfies the 
uniform bound
\begin{equation}
\label{equ_thm_S_approx_manifold}
\norm{\mathcal{S}_{k_{\mathcal{M}}, \theta}[k_{\mathcal{M}}(\cdot, x)]
-
\mathcal{S}[k_\mathcal{M}](\cdot, x)}_{\mathcal{H}^1(\mathcal{M})} < \delta,
\quad x \in \mathcal{M},
\end{equation}
and the $k_\mathcal{M}$ kernel approximation of function $\tilde{f}$ satisfies
\begin{equation}
\label{eq:operator_extension_interpolation_manifold}
\norm{\tilde f - f}_{\mathcal{H}^1(\mathcal{M})} < \epsilon.
\end{equation} 
If the native space $\mathcal{H}_k(\R^d)$ is norm-equivalent to
$\mathcal{H}^s(\R^d)$ with $s - (d-m)/2 > m/2$ and $\lceil s - (d-m)/2 \rceil
\geq 1$, then  the following bound holds,
\begin{equation}
\label{eq:thm_ext_acc_M}
\norm{\mathcal{S}_{k_\mathcal{M}, \theta}^e[\tilde f] - \mathcal{S}[f]}_{\mathcal{H}^1(\mathcal{M})}
\leq C_1 \epsilon + C_2 \delta.
\end{equation}
\end{theorem}

These results show how we can extend neural operators in the manifold setting and ensure
accuracy in the $\mathcal{H}^1$ norm. The Theorem~\ref{thm:op_extend_manifold} 
also shows to achieve accuracy we can avoid in the manifold setting 
the cumbersome task of constructing special translation invariant or isotropic 
kernels on $\mathcal{M}$.  We can instead leverage properties of kernels $k$ 
defined on the ambient space $\R^d$ and then restrict them to obtain $k_{\mathcal{M}}$
for $\mathcal{M}$, see Figure~\ref{fig:restricted_kernel}.  These theorems
also show for the errors of the extended operators it is sufficient to control
the accuracy $\delta$ of the training for the learned operator 
and the accuracy $\epsilon$ in the kernel approximations. 
We discuss in more detail below these results 
and characterize further how the native spaces
$\mathcal{H}_k(\Omega)$ and $\mathcal{H}_k(\mathcal{M})$ 
depend on the choices for the kernels. 

\section{Native Spaces for Kernel Approximations.}
\label{sec:kernel_interpolation}

We consider on $\R^d$ kernels $k(\cdot,\cdot)$ that are real-valued, symmetric, 
positive-definite, and translation invariant. In this case, the kernels can be expressed as 
$k(x, y) = \Phi(\sigma\cdot(x - y)) \in L^1(\R^d) \cap C(\R^d)$ on $\R^d$. For brevity in
the notation, we will suppress in many contexts the shape parameter $\sigma$, and take it fixed 
throughout the discussions. When we are able to use the Fourier transform $\hat \Phi$ of $\Phi$,
we can characterize the native space $\mathcal{H}_k(\R^d)$ as the collection 
of functions~\cite{Wendland2004} 
\begin{equation}
    \label{eq:native_space_fourier}
    \mathcal{H}_k(\R^d) = \set{f \in L^2(\R^d) 
    \cap C(\R^d) : \frac{\hat f}{\sqrt{\hat \Phi}} \in L^2(\R^d)}.
\end{equation}
When $s > d/2$ and 
\begin{equation}
    \label{eq:fourier_sobolev_equivalence}
    \hat \Phi(\omega) \sim \pn{1 + \norm{\omega}_2^2}^{-s},
\end{equation}
we have that the native space $\mathcal{H}_k(\R^d)$ is norm-equivalent to the Sobolev Space 
$\mathcal{H}^s(\R^d)$~\cite{Santin2018}.  The set of kernels that we will consider here
can be found in Table~\ref{table:sobolev_kernels}.   This includes the widely-used Gaussian 
kernel $\Phi(r) = \exp(-r^2)$. While the Gaussian kernel does not have a native 
space $\mathcal{H}_k$ that is norm-equivalent to any $\mathcal{H}^s$, it's native space is contained
as $\mathcal{H}_k \subset \mathcal{H}^s$ for all positive $s$. This allows us
to achieve related error bounds for functions in the native space of Gaussian
kernels~\cite{Wendland2004}. 

For manifolds and other geometric structures $\mathcal{M}$, it can be shown that norm
equivalence of the native space to Sobolev space also holds provided 
$\mathcal{M} \subset \R^d$ has Lipschitz boundary. In this case, the
restriction of $k$ to $\mathcal{M}$ yields the native space 
$\mathcal{H}_k(\mathcal{M}) = \mathcal{H}^s(\mathcal{M})$, \cite{Wendland2004}. To 
simplify the notation, we will often denote the domain as $\Omega = \mathcal{M}$ 
when discussing the manifold case.

\begin{figure}[t]
\centering
\includegraphics[width=0.9\linewidth]{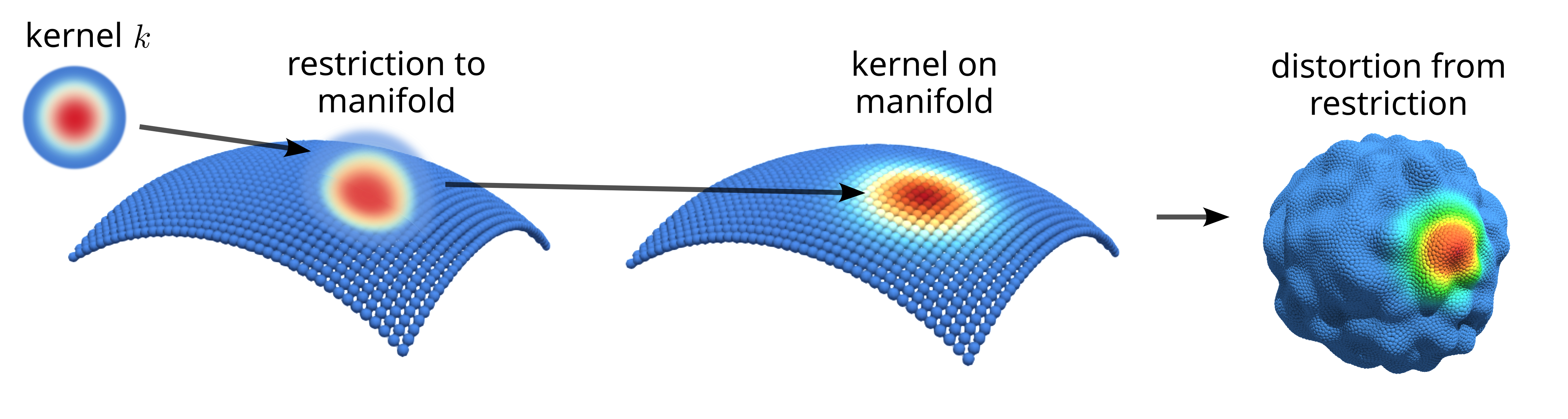}
\caption{\textbf{Kernels Restricted to Manifolds $k_{\mathcal{M}}$.} We develop theory for kernel
approximation based on restricting kernels from the bulk ambient space $k(x,y)$ to
a manifold surface $k_{\mathcal{M}}(x,y)$. Even when the bulk kernel $k$ has 
nice properties and symmetries, such as being a radial basis function, the 
restricted kernels $k_{\mathcal{M}}$ may not inherit these attributes. 
The $k_{\mathcal{M}}$ may no 
longer be radial symmetric 
$k_{\mathcal{M}}(x,y) \neq k_{\mathcal{M}}(|x - y|)$ or translation invariant 
$k_{\mathcal{M}}(x,y) \neq k_{\mathcal{M}}(x,y+\Delta)$.}
\label{fig:restricted_kernel}
\end{figure}

We consider $\mathcal{M}$ that are $m$-dimensional embedded smooth sub-manifolds
in $\R^d$ with $m < d$.  We use restricted kernels $\tilde k: \mathcal{M}
\times \mathcal{M} \to \R$ given by
\begin{equation}
\label{eq:restricted_kernel} \tilde k(x, y) = k_{\mathcal{M}}(x, y) = k \vert_{\mathcal{M}\times \mathcal{M}}(x,y),  \quad x, y \in \mathcal{M}.
\end{equation}
In this case, we have the native space $\mathcal{H}_{\tilde k}(\mathcal{M}) =
\mathcal{H}^{s - (d-k)/2}(\mathcal{M})$, \cite{Fuselier2012}. We can see 
there is a decrease in smoothness which corresponds to a $1/p$ penalty for 
each co-dimension $d-k$ of the manifold. The most common case we will consider is when $p=2$. 
This follows from extending trace operators to Sobolev spaces, restricting
functions from $\R^d$ to $\mathcal{M}$, see~\cite{Besov1978}.  
This is an interesting result since it is what allows for using the kernel $k$ from the
embedding space and restricting. This allows us to avoid the potentially cumbersome 
task of needing to design a kernel intrinsic to the manifold. 
An example of restricting a kernel $k$ on $\R^2$ to a $2$-dimensional manifold 
can be found in Figure~\ref{fig:restricted_kernel}. 

The results allow for using a wide range of accessible kernels $k$ to 
generate Sobolev native spaces. This allows our operator extension methods 
to be applied for a wide set of operators and functions. This also provides
results for the accuracy of these methods both for the Euclidean case 
and the Non-Euclidean manifold setting.

\begin{table}[ht!]
\begin{center}
\resizebox{0.86\linewidth}{!}{
\setlength{\tabcolsep}{5pt}
\renewcommand{\arraystretch}{1.6}
\begin{tabular}{|l|l|l|l|}
\hline
\rowcolor{black!20!white}
\textbf{Kernel}  & \textbf{Radial Function} $\bsy{\Phi(r)}$ & $\bsy{\mathcal{H}^s(\R^d)}$ &
$\bsy{\mathcal{H}^s(\mathcal{M})}$\\
\hline
\rowcolor{black!5!white}
Gaussian 
& $\exp\left(-r^2\right)$ 
& N/A 
& N/A \\
\rowcolor{black!12!white}
Mat\'ern Basic, $\nu = \frac{1}{2}$  
& $\exp\left(-r\right)$ 
&$s = \frac{d + 1}{2}$ 
& $s=\frac{3}{2}$ \\
\rowcolor{black!5!white}
Mat\'ern Linear, $\nu = \frac{3}{2}$ 
& $(1 +\sqrt{3}r)e^{-\sqrt{3}r}$ 
& $s = \frac{d + 3}{2}$
& $s=\frac{5}{2}$ \\
\rowcolor{black!12!white}
Mat\'ern Quadratic, $\nu = \frac{5}{2}$ 
& $(1 + \sqrt{5}r + \frac{5}{3}r^2)e^{-\sqrt{5}r}$ 
& $s = \frac{d + 5}{2}$ 
& $s=\frac{7}{2}$ \\
\rowcolor{black!5!white}
Wendland, $k=0$, $l= \lfloor d/2 \rfloor + 1$
& $(1 - r)_+^l$ 
& $s = \frac{d+1}{2}$ 
& $s=\frac{3}{2}$ \\
\rowcolor{black!12!white}
Wendland, $k=1$, $l= \lfloor d/2 \rfloor + 2$
& $(1 - r)_+^{l+1} ((l+1)r + 1)$ 
& $s = \frac{d+3}{2}$
& $s=\frac{5}{2}$ \\
\rowcolor{black!5!white}
Wendland, $k=2$, $l= \lfloor d/2 \rfloor + 3$
& $(1 - r)_+^{l+2}((l + 1)(l+3)r^2 + 3(l+2)r + 3)$ 
& $s = \frac{d+5}{2}$
& $s=\frac{7}{2}$ \\
\hline
\end{tabular}
}
\caption{\textbf{Kernels and RKHS Equivalent Native Spaces.} For the different 
kernels $k(x,y) = \Phi(|x - y|)$ and RKHS $\mathcal{H}_k$ generated by $k$, 
we give the equivalent native spaces. In the cases we consider,
these are Sobolev spaces $\mathcal{H}^s$, \cite{Santin2018}. For an input $\mb{x} \in
\mathbb{R}^d$, we take $r = \norm{\mb{x}}_2$. We also consider $\mb{x} \in
\mathcal{M}$ restricted to a manifold surface $\mathcal{M} \subset \R^3$. For
kernels obtained by restricting the free-space kernel $k(x,y)$, 
we also give the corresponding Sobolev space $\mathcal{H}^s(\mathcal{M})$.  }
\label{table:sobolev_kernels}
\end{center}
\end{table}

\begin{figure}[h]
\centering
\includegraphics[width=0.99\linewidth]{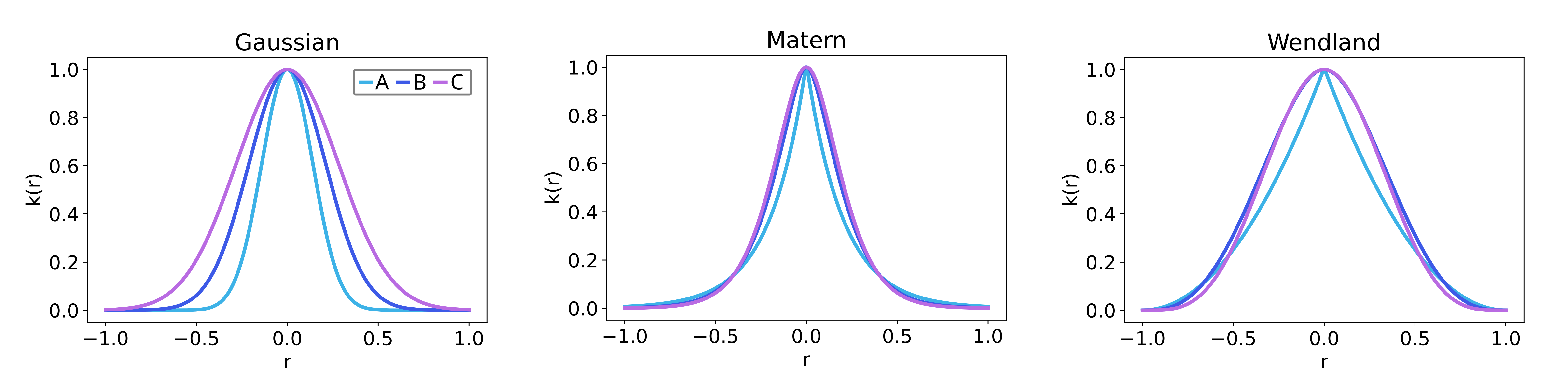}
\caption{\textbf{Kernels.} We show a few different kernels used in our 
approximations and comparison studies.  For the Gaussian kernels, 
we use $\Phi(\sigma \cdot r)$ with values for 
$\sigma = \ell_0^{-1}$ with $\ell_0 = 0.2, 0.3,0.4$ for 
the cases labeled $A$, $B$, $C$. For the Mat\'ern kernels, 
the $A$, $B$, $C$ cases correspond to the Basic, Linear,
and Quadratic kernels with the parameters shown in Table~\ref{table:sobolev_kernels}. For 
the Wendland kernels, we consider the cases with $k=0$, $k=1$, $k=2$ with 
parameters in Table~\ref{table:sobolev_kernels}. 
}
\label{fig:kernels}
\end{figure}

\subsection{Function Approximation using Kernel Methods}
\label{subsec:kernel_interpolation}
For a given kernel $k$ we seek approximations of a function $f$
in the RKHS $\mathcal{H}_k$ generated by $k$.  
We formulate the approximation as the solution to the optimization problem
\begin{equation}
    \label{eq:kernel_best_approximant_rep}
    s_f^{\lambda} = \argmin_{s \in \mathcal{H}_k(\Omega)} 
    \sum_{i=1}^N \pn{f(x_i) - s(x_i)}^2 + \lambda \|s\|_{\mathcal{H}_k(\Omega)}.
\end{equation}
This gives an approximation for a given 
set of $N$ points $x_i \in X_N \subset \Omega$ and 
given function $f \in \mathcal{H}_k(\Omega)$.  
By the Kernel Representer Theorem 
any objective function of the form 
$\ell(s(x_1),\ldots,s(x_N)) + \lambda \|s\|_{\mathcal{H}_k}$
always has a minimizing solution of the form 
$s_f^\lambda(x) = \sum_{i=1}^N \alpha_i^{\lambda} k(x, x_i)$,
\cite{Scholkopf2001}.  This is useful since it 
only involves a finite number of kernel 
evaluations at $x_i \in X_N$.
While in principle there can also be other 
minimizing functions in $\mathcal{H}_k$, we 
can always achieve the same minimum value
of the objective using a function of this form.
In our notation, we let 
$s_f = \lim_{\lambda \rightarrow 0} s_f^{\lambda}$. 

The quality of the function
approximation of $s_f$ to $f$ depends on the quality of the point
sampling $X_N$. It can be shown the kernel approximation 
in equation~\ref{eq:kernel_best_approximant_rep}
satisfies at each point $x \in \Omega$ the bound~\cite{schaback2023small,Santin2018} with 
\begin{equation}
\abs{f(x) - s_f(x)} \leq P_{X_N}(x) \norm{f}_{\mathcal{H}_k(\Omega)}, 
\quad x \in \Omega.
\end{equation}
The function $P_{X_N}$ is called the Power Function of $k$ with respect to
$X_N$ and can be found using the Newton basis $\set{l_i(x) \in V(X_N) :
l_i(x_j) = \delta_{ij}}$, 
\begin{equation}
\label{eq:power_function}
P_{X_N}(x) \coloneqq \norm{k(\cdot, x) - \sum_{j = 1}^N k(\cdot, x_j)l_j(x)}_{\mathcal{H}_k(\Omega)}, 
\quad x \in \Omega.
\end{equation}
The space of functions is defined by 
$V(X_N) := \{h(\cdot) \;|\; h(x) = \sum_i \alpha_i k(x,x_i), \;\; x_i \in X_N\}$.
In practice, while computing $P_{X_N}(x)$ is often more difficult than solving the 
approximation problem, it can provide insights into the key factors impacting 
approximations.  This can be used to obtain bounds 
on $\abs{f(x) - s_f(x)}$ based on the distribution of the points in $X_N$. 
A useful characteristic is the \textit{fill distance} 
$h_{X_N}$ of $X_N$ defined as 
\begin{equation}
\label{eq:fill_distance}
h_{X_N} \coloneqq \sup_{x \in \Omega} \min_{x_j \in X_N} \norm{x - x_j}_2^2. 
\end{equation}
This can be used in the error analysis of $f(x)$ and $s_f(x)$ and their
derivatives~\cite{Wendland2004}. The $P_{X_N}$ depends on the
quality of the sampling $X_N$. Related to these approaches, we state two
theorems below from the literature for the Euclidean $\R^d$ setting and the
Non-Euclidean embedded sub-manifold $\mathcal{M}$
setting~\cite{Wendland2004,Fuselier2012}.

\begin{theorem}
    \label{theorem:sobolev_approximation_theorem}
    Suppose that $\Omega \subset \R^d$ is bounded, has a Lipschitz boundary, and 
    satisfies an interior cone condition. Let $k$ be a positive definite 
    kernel satisfying equation~\ref{eq:fourier_sobolev_equivalence} and $\mu \in \N_0$
    satisfying $\lceil s \rceil - \mu - 1 > d/2$. Then for any set $X_N \subset \Omega$
    with $h_{X_N} \leq h_0$ where $h_0$ depends on $s, \Omega$, and any 
    $f \in \mathcal{H}_k(\Omega)$, the error between $f$ and its interpolant $s_f$ on
    $X_N$ can be bounded by 
    \begin{equation}
        \label{eq:sobolev_interpolant_error}
        \norm{f - s_f}_{H^\mu(\Omega)} \leq C h_{X_N}^{s - \mu} \norm{f}_{\mathcal{H}^s(\Omega)}.
    \end{equation}
\end{theorem}
\begin{theorem}
    \label{theorem:sobolev_approximation_theorem_manifold}
Let $\mathcal{M}$ be a $m$-dimensional sub-manifold of $\R^d$ and $k$ be a positive
definite kernel satisfying equation~\ref{eq:fourier_sobolev_equivalence}, and 
define $\tilde k$ by restricting $k$ to $\mathcal{M}$. Let $\tilde{s} = s - (d-m)/2$,
assume $\tilde{s} > m / 2$
and let $\mu \in \N$ with $0 \leq \mu \leq \lceil \tilde{s} \rceil - 1$. Then there 
is a constant $h_\mathcal{M}$ such that if a finite node set $X_N \subset \mathcal{M}$
satisfies $h_{X_N, \mathcal{M}} \leq h_{\mathcal{M}}$, then for all 
$f \in \mathcal{H}_{\tilde k}(\mathcal{M})$ we have
\begin{equation}
    \label{eq:manifold_sobolev_approximation}
    \norm{f - s_f}_{\mathcal{H}^\mu(\mathcal{M})} 
\leq C h_{X_N, \mathcal{M}}^{\tilde{s} - \mu} 
\norm{f}_{\mathcal{H}_{\tilde k}(\mathcal{M})}.
\end{equation}
\end{theorem}
The results of Theorem~\ref{theorem:sobolev_approximation_theorem_manifold}
show that by using our projected kernels $k_{\mathcal{M}}$ we can still achieve 
a comparable level of accuracy and
convergence as in Theorem~\ref{theorem:sobolev_approximation_theorem}. This
provides an approach that gives the same convergence rate as if we had designed an
isotropic and translation-invariant kernel for the manifold $\mathcal{M}$.
This allows us to construct radial kernels $k$ on the
ambient space $\R^d$ of dimension $d$ and then by restricting
them to obtain $k_{\mathcal{M}}$ still be assured of good 
approximation properties
for the manifold 
$\mathcal{M}$ of dimension $m$. The $\tilde{s} = s - (d-m)/2$
gives the reduction $d - m$ in smoothness of the Sobolev space 
with $p=2$ which arises whenever one reduces the spatial dimension
of the domain from $d$ to $m$. 

The fill distance $h_{X_N, \mathcal{M}}$ in
Theorem~\ref{theorem:sobolev_approximation_theorem_manifold} is with respect to
the intrinsic metric $d_\mathcal{M}$ of the manifold $\mathcal{M}$.  We remark
that having the results stated in terms of $h_{X_N, \mathcal{M}}$ in practice
is not an issue since for smooth compact manifolds it can be bounded by 
the ambient fill distance $h_{X_N,\R^d}$ and these become similar as the 
density of points increases~\cite{Wendland2004,Fuselier2012}.  
This can also be related to the fill
distance considered in the local coordinate charts $h_{X_N \cap U_j, U_j}$ for
$U_j \subset \R^m$~\cite{Fuselier2012}. These results 
can be used to provide bounds for our kernel approximations of functions 
in our extension methods. It also should be mentioned that in practice 
the choice of kernel also plays an important role in the 
level of numerical accuracy and conditioning of the approximations. 
We discuss in more detail below these and other practical aspects 
of kernel approximation.

\subsection{Function Approximation using Regularized Kernel Methods}
\label{subsec:regularized_kernel_interpolation}

The approximations obtained from equation~\ref{eq:kernel_best_approximant_rep}
in the limit $\lambda \rightarrow 0$  
can become ill-conditioned depending on the choice of kernel and the 
quality of the point sampling $X_N$. 
The conditioning also tends to become worse as the number of sample points
 $N$ increases.  This can arise from sample points that are nearly identical, 
overlap of the kernels for points on length-scales smaller than $\sigma^{-1}$,
and increasing density as $N$ increases.  
One approach to mitigate the ill-conditioning is to relax the $\lambda \rightarrow 0$ 
conditions for residual minimization and instead solve the regularized problem
with $\lambda > 0$,
\begin{equation}
    \label{eq:regularized_interpolation}
    s_f^{\lambda} = \argmin_{s \in \mathcal{H}_k(\Omega)} 
    \sum_{i=1}^N \pn{s(x_i) - f(x_i)}^2 + \lambda \norm{s}_{\mathcal{H}_k(\Omega)}^2.
\end{equation}
The $\lambda$ provides the strength of the regularization.
When $\lambda = 0$, we recover the residual minimization problem which 
corresponds to the interpolation case. 
In practice to find the coefficients of the solution function, 
we set to zero the gradient of the objective function 
in equation~\ref{eq:regularized_interpolation} to obtain the linear problem
\begin{equation}
    \label{eq:representer}
    (A + \lambda I) \bsy{\alpha} = \mb{b}, \quad [\mb{b}]_i = b_i = f(x_i),
\end{equation}
where $A$ is the Gram Matrix with entries $k(x_i,x_j)$ and $\bsy{\alpha}$ is
the vector of coefficients with $\alpha_i = [\bsy{\alpha}]_i$. When $\lambda > 0$ 
this form of the solution is unique. 
By adding the regularization term we can ensure that the condition 
number $\kappa(A + \lambda I)$ satisfies  
\begin{equation}
\kappa(A + \lambda I) = \frac{\lambda + \lambda_{\max}(A)}{\lambda + \lambda_{\min}(A)}
\leq 1 + \lambda_{\max}(A) / \lambda.
\end{equation}
This provides a useful bound that no longer depends on the smallest eigenvalue $\lambda_{\min}(A)$. 
It is known that while the largest eigenvalue has an upper bound of $N \Phi(0)$ 
that grows according to the number of points, the smallest 
eigenvalue also can decay rapidly $\rightarrow 0$,~\cite{Wendland2004}.
The utility of the regularization can be seen 
since the behavior of the smallest eigenvalue $\lambda_{\min}(A)$
depends on the separation radius $q_{X_N}$ defined by
\begin{equation}
    \label{eq:separation_radius}
    q_{X_N} \coloneqq \min_{i \neq j} \frac{1}{2}\norm{x_i - x_j}_2.
\end{equation} 
This $q_{X_N}$ plays a central role in determining how small the eigenvalues become,
with $q_{X_N}= 0$ yielding a zero eigenvalue. 
For example, in the case of a Gaussian
Kernel with $\Phi(r) = \exp(-\sigma r^2)$, we have the lower bound~\cite{Wendland2004}
\begin{equation}
    \label{eq:gaussian_minmal_eigenvalue}
    \lambda_{\min}(A) \geq C (2 \sigma)^{-d/2} \exp\pn{-C'/(q_{X_N}^2 \sigma)} q_{X_N}^{-d}.
\end{equation}
While this is a lower bound, for Gaussian kernels this same behavior often occurs 
for $\lambda_{\min}$ for typical point samplings encountered in applications. 
In practice, the smallest eigenvalue often decays rapidly toward $0$ 
as the separation distance $q_{X_N}$ becomes small.  For a few numerical studies,
we show condition numbers in Appendix~\ref{sec:appendix_a}. Similar to 
equation~\ref{eq:sobolev_interpolant_error}, we see a bound in the 
regularized case also can be obtained of the form~\cite{Wendland2005}
\begin{equation}
    \label{eq:interpolant_error_regularized}
    \norm{D^a s_f^\lambda - D^a f}_{L^2(\Omega)} \leq 
    C \pn{h_{X_N}^{\ell - \abs{a}} + h_{X_N}^{-\abs{a}}\sqrt{\lambda}}\norm{f}_{\mathcal{H}_k(\Omega)}.
\end{equation}
The $D^a f$ indicates the derivative of the function $f$ with partial derivatives in the 
$i^{th}$ direction of order $a_i = [a]_i$.
Here, the constant $C > 0$ depends on $\Omega, k$ and the smoothness $\ell$ of the kernel. 
For solving the regularized problem this also suggests balancing the penalty with the 
fill-distance $h$ by taking $\lambda = \lambda(h) \leq h_{X_N}^{2\ell}$.  
In this case, we obtain the bound 
\begin{equation}
\label{eq:interpolant_error_regularized2}
\norm{s_f^{\lambda(h)} - f}_{\mathcal{H}^{\abs{a}}(\Omega)} \leq \tilde{C}
h_{X_N}^{\ell - \abs{a}}\norm{f}_{\mathcal{H}_k(\Omega)}.
\end{equation}
We remark that this bound holds for the Sobolev-norm 
$\norm{\cdot}_{\mathcal{H}^{\abs{a}}(\Omega)}$ 
and these 
results provide ways for us to control the accuracy of
the kernel approximation both 
of functions $f$ and their derivatives $D^a f$ in the Native Spaces $\mathcal{H}_k$ 
in terms of the fill-distance $h_{X_N}$ of the point samplings $X_N$ and 
choice of kernel $k(\cdot,\cdot)$.  As we discuss in more detail 
and show in our empirical studies, the choice of kernel and parameters 
can have a significant impact on the performance of the approximations 
and extension methods.

\section{Proofs for the Kernel Extension Theorems~\ref{thm:op_extend} 
and~\ref{thm:op_extend_manifold}.}
\label{sec:operator_extension_proof_sec} 
We now use the results of the previous sections to provide proofs of 
our results in Theorem~\ref{thm:op_extend} and 
Theorem~\ref{thm:op_extend_manifold}.  These results leverage the 
properties of kernel approximation of the input functions.  We also
utilize the relationship between the Sobolev Native Spaces generated by
the unrestricted Euclidean kernels $k(\cdot,\cdot)$ 
and Non-Euclidean manifold kernels $k_{\mathcal{M}}(\cdot,\cdot)$.
We also characterize the behaviors of the associated Sobolev Native 
Spaces these kernels generate in the manifold setting. 

\subsection{Proof of Theorem~\ref{thm:op_extend}.}
\label{sec:operator_extension_proof}
We now establish the bounds given by Theorem~\ref{thm:op_extend}. 
\begin{proof}
Fix $\epsilon > 0, f \in \mathcal{H}^s(\Omega)$, and $\delta > 0$. As
stated in the theorem, we assume training can be
performed to obtain $\mathcal{S}_{k, \theta}$ 
acting on kernel input functions $k$ having an accuracy $\delta$ 
for a target operator $\mathcal{S}$ 
satisfying
\begin{equation} \label{equ_thm_S_approx_pf_2_1}
\norm{\mathcal{S}[k(\cdot, x)] - \mathcal{S}_{k, \theta}[k(\cdot,
x)]}_{\mathcal{H}^1(\Omega)} < \delta, \quad x \in \Omega.  
\end{equation}
Now using Theorem~\ref{theorem:sobolev_approximation_theorem}, we can choose a
set of points $X_N = \set{x_i}_{i=1}^N$ that have a fill distance $h_{X_N}$.
We choose the fill distance sufficiently small so that $h_{X_N} \leq h_0$ and 
\begin{equation}
\label{eq:fill_distance_small} h_{X_N}^{s-1} < \frac{\epsilon}{C
\norm{f}_{\mathcal{H}^s(\Omega)}}.
\end{equation}
We use the kernel approximation $s_f$ for $f$ that minimizes
equation~\ref{eq:kernel_best_approximant_rep}. This can be represented as $s_f(x) =
\sum_{i=1}^N \alpha_i
k(x, x_i)$ and satisfies $\norm{f - s_f}_{\mathcal{H}^1(\Omega)} <
\epsilon$ by equation~\ref{eq:interpolant_error_regularized2} and 
equation~\ref{eq:fill_distance_small}.  By setting 
$\tilde f = s_f$ and using the linearity and boundedness of $\mathcal{S}$, we have    
\begin{eqnarray}
\label{equ_pf_2_1_S_ineq_01}
\norm{\mathcal{S}[f] - \mathcal{S}_{k, \theta}^e[s_f]}_{\mathcal{H}^1(\Omega)}
&\leq& \norm{\mathcal{S}[f] - \mathcal{S}[s_f]}_{\mathcal{H}^1(\Omega)}
+ \norm{\mathcal{S}[s_f] - \mathcal{S}_{k, \theta}^e[s_f]}_{\mathcal{H}^1(\Omega)}\\
\label{equ_pf_2_1_S_ineq_02}
&\leq& \norm{\mathcal{S}}_{\mathcal{H}^1(\Omega)} \norm{f - s_f}_{\mathcal{H}^1(\Omega)}
+ \norm{\sum_{i=1}^N \alpha_i(\mathcal{S} - \mathcal{S}_{k, \theta})[k(\cdot,
x_i)]}_{\mathcal{H}^1(\Omega)}\\
\label{equ_pf_2_1_S_ineq_03}
&\leq&  \norm{\mathcal{S}}_{\mathcal{H}^1(\Omega)} \norm{f - s_f}_{\mathcal{H}^1(\Omega)}
+ \sum_{i=1}^N \abs{\alpha_i}
\norm{(\mathcal{S} - \mathcal{S}_{k, \theta})[k(\cdot, x_i)]}_{\mathcal{H}^1(\Omega)} \\
\label{equ_pf_2_1_S_ineq_04}
& \leq & C_1 \epsilon + C_2 \delta.
\end{eqnarray}
The constants are given by $C_1 = \norm{\mathcal{S}}_{\mathcal{H}^1(\Omega)}$ and 
$C_2 = \norm{\alpha}_{\ell_1} = \sum_{i=1}^N |\alpha_i|$.
We obtain equation~\ref{equ_pf_2_1_S_ineq_01} 
from the triangle inequality, and 
equation~\ref{equ_pf_2_1_S_ineq_02} by using 
boundedness of $\mathcal{S}$ and the specific form of $\tilde{f}$.
We then obtain equation~\ref{equ_pf_2_1_S_ineq_03} by properties of
the norm. The final bound in equation~\ref{equ_pf_2_1_S_ineq_04} follows
from the approximation accuracy of $\tilde{f}$ in
equation~\ref{eq:operator_extension_interpolation} and the 
uniform accuracy of $\mathcal{S}_{k, \theta}$ 
from equation~\ref{equ_thm_S_approx}. This provides the bound on the 
accuracy of the operator extension $\mathcal{S}_{k, \theta}^e$ given in 
equation~\ref{eq:thm_ext_acc}.
\end{proof}

\noindent

\subsection{Proof of Theorem~\ref{thm:op_extend_manifold}.}

We next discuss how to prove Theorem~\ref{thm:op_extend_manifold}.  This
requires we obtain estimates when restricting approximations to the
sub-manifold $\mathcal{M}$. While there are similarities to the Euclidean case,
the restrictions result in kernels that are no longer radial symmetric $k(x,y)
\neq k(|x - y|)$ or satisfy the translation invariance properties $k(\cdot,y)
\neq k(\cdot,y+ \Delta)$ for all $x,y$.  The approximations also now only
involve sample points from the sub-manifold $x_i \in \mathcal{M}$.  As part of handling
these issues, we leverage
Theorem~\ref{theorem:sobolev_approximation_theorem_manifold}.  We use $s_f$ to
approximate functions $f$ on $\mathcal{M}$ by minimizing
equation~\ref{eq:kernel_best_approximant_rep}. From
Theorem~\ref{theorem:sobolev_approximation_theorem_manifold}, we have the bound 
\begin{equation}
\label{eq:kernel_approximation_manifold}
\norm{f - s_f}_{H^\mu(\M)} \leq 
C h_{X, \M}^{s - \mu - (d - m)/2}\norm{f}_{\mathcal{H}^{s - (d-m)/2}(\M)}.
\end{equation}
We use these results in establishing Theorem~\ref{thm:op_extend_manifold}.
\begin{proof}
From equation~\ref{eq:manifold_sobolev_approximation}, 
we obtain $\|\tilde{f} - f\|_{\mathcal{H}^1(\mathcal{M})} < \epsilon$ by considering the case  when $\mu = 1$ 
with $\tilde{f} = s_f$ when the fill-distance $h_{X_N, \mathcal{M}}$ for points 
$X_N \subset \mathcal{M}$ satisfies 
\begin{equation}
    \label{eq:fill_distance_small_manifold}
h_{X_N, \mathcal{M}}^{s - (d-m)/2 - 1} < \frac{\epsilon}{C \norm{f}_{\mathcal{H}_{\tilde k}(\mathcal{M})}}.
\end{equation}
We remark that a significant difference with the previous result 
in equation~\ref{eq:fill_distance_small} is that the exponent is augmented now to 
$s - (d-m)/2$.  The $d - m$ is the 
co-dimension of the sub-manifold $\mathcal{M}$ when embedded in
$\R^d$.   This arises from
the distortions that can occur in
the kernel approximations when they are restricted 
to the lower dimensional sub-manifold, as proven in~\cite{Fuselier2012}.  

From these considerations and using the linearity and boundedness of $\mathcal{S}$, 
we obtain
\begin{eqnarray}
\label{equ_pf_2_2_S_ineq_01}
\norm{\mathcal{S}[f] - \mathcal{S}_{k, \theta}^e[s_f]}_{\mathcal{H}^1(\mathcal{M})}
&\leq& \norm{\mathcal{S}[f] - \mathcal{S}[s_f]}_{\mathcal{H}^1(\mathcal{M})}
+ \norm{\mathcal{S}[s_f] - \mathcal{S}_{k, \theta}^e[s_f]}_{\mathcal{H}^1(\mathcal{M})}\\
\label{equ_pf_2_2_S_ineq_02}
&\leq& \norm{\mathcal{S}}_{\mathcal{H}^1(\mathcal{M})} \norm{f - s_f}_{\mathcal{H}^1(\mathcal{M})}
+ \norm{\sum_{i=1}^N \alpha_i(\mathcal{S} - \mathcal{S}_{k, \theta})[k(\cdot, x_i)]}_{\mathcal{H}^1(\mathcal{M})}\\
\label{equ_pf_2_2_S_ineq_03}
&\leq&  \norm{\mathcal{S}}_{\mathcal{H}^1(\mathcal{M})} \norm{f - s_f}_{\mathcal{H}^1(\mathcal{M})}
+ \sum_{i=1}^N \abs{\alpha_i}
\norm{(\mathcal{S} - \mathcal{S}_{k, \theta})[k(\cdot, x_i)]}_{\mathcal{H}^1(\mathcal{M})} \\
\label{equ_pf_2_2_S_ineq_04}
& \leq & C_1 \epsilon + C_2 \delta.
\end{eqnarray}
The equation~\ref{equ_pf_2_2_S_ineq_01} is obtained from the triangle
inequality.  We obtain the equation~\ref{equ_pf_2_2_S_ineq_02} by using
boundness of the operator $\mathcal{S}$ and the specific form of $\tilde{f} = s_f$.  
This yields
equation~\ref{equ_pf_2_2_S_ineq_03} by properties of the $\mathcal{H}^1$ norm. We next use the
approximation accuracy of $\tilde{f}$ on the sub-manifold given by
equation~\ref{eq:operator_extension_interpolation_manifold} and the uniform
accuracy of $\mathcal{S}_{k, \theta}$ from
equation~\ref{equ_thm_S_approx_manifold}.  Together, this yields the final
equation~\ref{equ_pf_2_2_S_ineq_04}. 
The constants are given by $C_1 = \norm{\mathcal{S}}_{\mathcal{H}^1(\mathcal{M})}$ and 
$C_2 = \norm{\alpha}_{\ell_1} = \sum_{i=1}^N |\alpha_i|$.
This provides the bound on the accuracy
of the operator extension given in equation~\ref{eq:thm_ext_acc_M}.

\end{proof}

This shows that the errors can be controlled for how well the restricted kernels 
$k_{\mathcal{M}}$ approximate the functions $f$ on the 
$m$-dimensional sub-manifold $\mathcal{M}$. 
The operator extensions can achieve 
accuracy provided we can get the errors sufficiently small
in training for the operator responses to $k_{\mathcal{M}}(\cdot,x_i)$
for $x_i \in X_N \subset \mathcal{M}$. The main difference with the 
Euclidean case is the increased need for smoothness in the unrestricted 
kernel, which is reduced under restriction. We also need
sufficiently small manifold fill-distance $h_{X_N,\mathcal{M}}$.
The theory shows that for sufficiently accurate training on kernel
responses we can obtain accurate operator extensions for more
general functions in the kernel's native space both in the Euclidean and 
Non-Euclidean settings. We next discuss in more detail 
practical ways to obtain effective training methods.

\section{Training Methods for the Neural Operators and Sobolev Loss}
\label{sec:gnp}
We discuss how the neural operators can be trained using loss functions
based on Sobolev norms. This ensures the training methods result 
in neural operators that capture both the target functions and
their derivatives. This allows for further regularizing the smoothness
of the learned operators and helps ensure that physically-relevant information
is retained during training.  To train our neural operators more
efficiently, we also introduce approximations for the kernel 
integral operators and ways to leverage separable factorizations.  

\subsection{Approximating Kernel Integral Operators using Separable Factors.}
A major computational expense both during training and evaluation of neural
operators is computing the kernel integral operators $\mathcal{K}[v]$.  For
these approximations, edge-based convolutions are widely used to approximate the
integrals in equation~\ref{kernel_integral}.  To improve performance, we 
develop more efficient alternatives using node-based convolutions by 
factoring kernels into separable forms.   

Consider how 
a neural operator transforms vector-valued functions $v_j(x)$ at the $j$-th layer of 
an $L$-layer network. The neural operator uses a pointwise linear 
function represented by a learned matrix $W_j$ and a learned kernel $k_{\theta,j}$
processed by a pointwise non-linear activation function $\sigma(\cdot)$. This
is performed using the mapping 
\begin{equation}
    \label{eq:operator_layer}
    v_j(x) \mapsto \sigma \left(W_jv_j(x) + \int_{\mathcal{D}} \kappa_{\theta,j}(x, y) v_j(y) \ dy\right).
\end{equation}
The kernel integral is computed on a domain $\mathcal{D}$. In our neural operators,
the domain is a ball of radius $r$ giving $\mathcal{D} = B_r(x)$. A common method
to approximate the integral is to use the edge-based message passing approach
\begin{equation}
    \label{eq:integral_operation}
    \int_{\mathcal{D}} \kappa_\theta(x, y) v(y) \ dy \approx 
    \frac{1}{\abs{\mathcal{N}(x)}} \sum_{y \in \mathcal{N}(x)} \kappa_\theta(x, y)v(y).
\end{equation}
The $\mathcal{N}(x)$ is the neighborhood of $x$ in the graph constructed with edges
between all points at distances less than $r$ from $x$. 
As the number of points grow, computing equation~\ref{eq:integral_operation} at every point
can become prohibitively expensive both in compute memory and compute time. We
introduced in our previous work a few ways to mitigate this by restricting the form 
and output of the kernels $\kappa_\theta$~\cite{AtzbergerQuackenbushGNP2024, Quackenbush2025}.
Here, we introduce further approaches for reducing these computational costs.

\begin{figure}[t]
\centering
\includegraphics[width=0.8\linewidth]{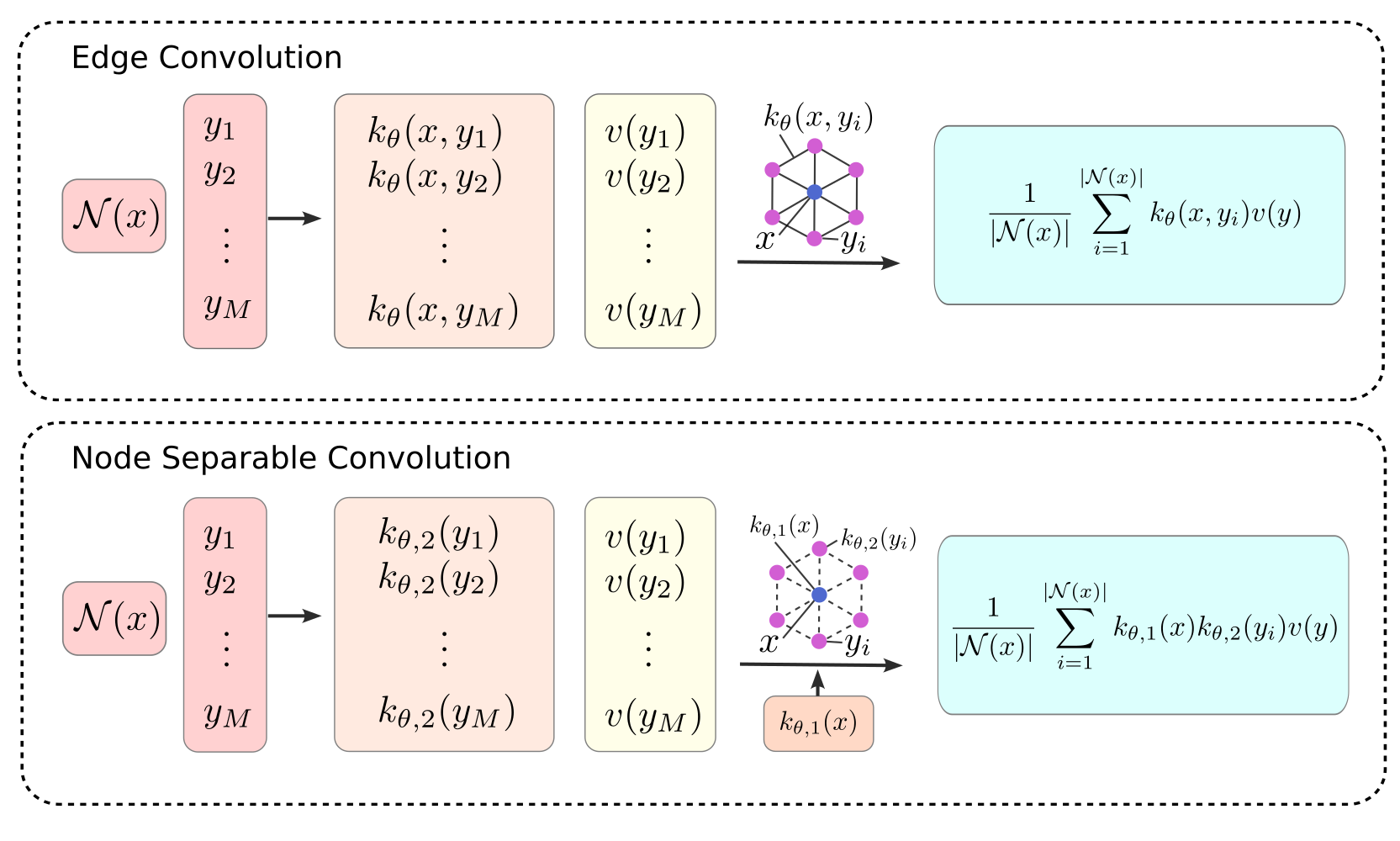}
\caption{\textbf{Kernel Integration: Methods for Improving Performance.} 
In neural operators a major computational cost during both training and evaluation
is to compute the kernel integral operations $\mathcal{K}[v]$.
An edge-conditioned convolution is widely used for this approximation,
but scales as $O(N^2)$ becoming prohibitively expensive as the density of points increases 
within the support of the kernel, \textit{(top)}. We develop more computationally efficient 
alternatives using node-conditioned convolutions based on factoring 
kernels into a separable form $k(x,y) = k_1(x)k_2(y)$. The node-based methods 
scale as $O(N)$ 
allowing for approximations using efficient gathering and scattering
operations greatly reducing computational costs, \textit{(bottom)}.
} 
\label{fig:separable_conv} 
\end{figure}

To simplify the message passing in equation~\ref{eq:integral_operation},
we consider here how to use separable representations of the kernel
network by factoring it as $k_\theta(x, y) = k_{\theta}^{(1)}(x)k_{\theta}^{(2)}(y)$. 
We find that this can be used to significantly reduce the costs to obtain 
more scalable node-conditioned gather-scatter operations.  Since
there is no longer an explicit dependence on the edge attributes, 
there are fewer compute steps and the methods are more amenable to 
parallelization. 
We also find that our separable architecture allows for training 
and evaluation on significantly larger point clouds without the 
need for sub-sampling edges in the graph, further improving accuracy
and computational performance. 

In more detail, the node-based message passing approach provides better scaling as the number of 
points $N$ grows. Each kernel network $k^{(i)}_\theta$ for $i=1,2$ need only be evaluated 
$O(N)$ times per layer. On the other hand, in the edge-based approach, the kernel 
$k_\theta$ must evaluate on the edges which grows as $O(N^2)$. In particular, 
for the point-cloud resolutions
we later consider in Section~\ref{sec:demonstration} with 
$N_{\text{train}} = 5000, N_{\text{test}}=10000$, the edge-based convolution
methods are so memory intensive we can not even perform the
computation for the Sobolev training with an 
Nvidia A40 GPU.
We also find when evaluating the node-based kernel using 
$N_{\text{test}}$ points takes 160ms compared to the edge-based case taking 
2400ms on an Nvidia A40 GPU.  Consistent with the asymptotic scaling in $N$,
these tests further show the impact of
the factorizations on the empirical efficiency gains.
The evaluation above already shows a reduction in the computational time of more than $10 \times$.  
As we show below and discuss in more detail, our separable neural operators 
are still able to achieve comparable accuracy to our 
previous edge-based results in~\cite{AtzbergerQuackenbushGNP2024}. 
The separable architectures allow
for processing larger point-clouds,
improve the efficiency of training, and speed up 
for pre-trained models the evaluation times for neural operators 
applied to new input functions.

\subsection{Sobolev Training of the Neural Operators}
\label{subsec:sobolev_training}
In our training of neural operators $\mathcal{G}_\theta$, we also consider the accuracy 
of the derivatives of the output functions by introducing Sobolev norms in
the loss function.  In the neural operators $\mathcal{G}_\theta$, we use within the first $L-1$ 
layers our separable kernel operators. These layers process the
coordinates $x$ along with the function values $f(x)$ based on 
\begin{equation}
    \label{eq:operator_layer_L-1}
    v_j(x, f(x)) \mapsto \sigma\pn{W_jv_j(x, f(x)) + 
    k_{\theta, 1}(x, f(x)) \int_{\mathcal{D}} k_{\theta,2}(y, f(y)) v_j(y, f(y)) \ dy}.
\end{equation}
This serves to encode the input function $f$ at the given points $x$. 
At the final operator layer $L$, we only process the coordinate data $x$
and omit the pointwise linear layer and activation, as 
also done in~\cite{Li2024a}.
 
During training, $\mathcal{G}_\theta$ learns to map $f(\cdot) \mapsto u(\cdot)$ 
for functions on the manifold $\mathcal{M}$. The $u$ is the solution for
the PDE having input $f$ in equation~\ref{eq:linear_pde}. To help ensure that $\mathcal{G}_\theta$ 
also learns the surface gradients in the embedding space coordinates, we also 
incorporate gradients in the loss as
\begin{equation}
\label{eq:sobolev_loss}
\mathcal{L}_{train}(u, \tilde u) =
\frac{\norm{u - \tilde u}_2}{\norm{u}_2}
+
\frac{\norm{\nabla_{\mathcal{M}}u - \nabla_{\mathcal{M}}\tilde u}_{2}}
{\norm{\nabla_{\mathcal{M}}u}_{2}}.
\end{equation}
We compute the surface gradients
\begin{equation}
    \label{eq:gradient_computation}
\nabla_{\mathcal{M}}u := \frac{\partial u}{\partial x^i}g^{ij} \bf{e}_j.
\end{equation}
This uses the local coordinate system $\mb{x} = (x_1,x_2)$ for the 
surface parameterization $\sigma(x_1,x_2)$ with  
$\bf{e}_j = \partial_j \sigma(x_1,x_2)$. This also uses the local inverse metric
tensor $g^{-1}$ with components $g^{ij}$. The $\bf{e}_j$ give the 
tangent vectors providing a local coordinate basis.  

To compute the derivatives of $u$ we use a pointwise operator network $\mathcal{Q}_\theta$.
This performs a projection of a function to the desired co-domain to obtain
$\tilde u(x)$
\begin{equation}
\label{eq:operator_layer_L}
\tilde u(x) = \mathcal{Q}_\theta\pn{k_{\theta, L}^{(1)}(x) \int_\mathcal{D}
k_{\theta,L}^{(2)}(y) v_{L-1}(y, f(y)) \ dy}.
\end{equation}
Derivatives can be computed efficiently using automatic differentiation since
these only involve the networks $\mathcal{Q}_\theta$ and $k_{\theta, L}^{(1)}(x)$.

We further perform a change of basis to represent functions in the embedding
space coordinates of $\R^3$. For training data, a local basis $\bf{e}_j$ 
and geometric quantities 
then can be computed locally using Generalized Moving Least Squares
(GMLS)~\cite{Atzberger2020,AtzbergerFPT2022,mirzaei2012generalized}. Since GMLS
uses a linear change of coordinates to form a local basis at each point, the
gradients computed from $\mathcal{G}_\theta$ can be easily transformed to the
local chart of the training data when computing 
equation~\ref{eq:gradient_computation}.
As we show below, our separable architectures for the neural operators perform
well with comparable accuracy to our previous edge-based operator results
in~\cite{AtzbergerQuackenbushGNP2024}. These architectures allow for significant
computational reductions in cost both during training and evaluation of the
learned operators on input functions.

\section{Results: Accuracy of the Neural Operators and Extension Methods}
\label{sec:demonstration}

We demonstrate the methods for empirically approximating the 
solution operators for geometric PDEs on manifolds. This requires
the neural operators to extract from the input functions 
and geometry the information relevant in constructing 
the solution functions. We show how our operator extension methods
can be used to obtain solution predictions from input functions beyond
those encountered during training. 
\begin{figure}
\centering
\includegraphics[width=1.\linewidth]{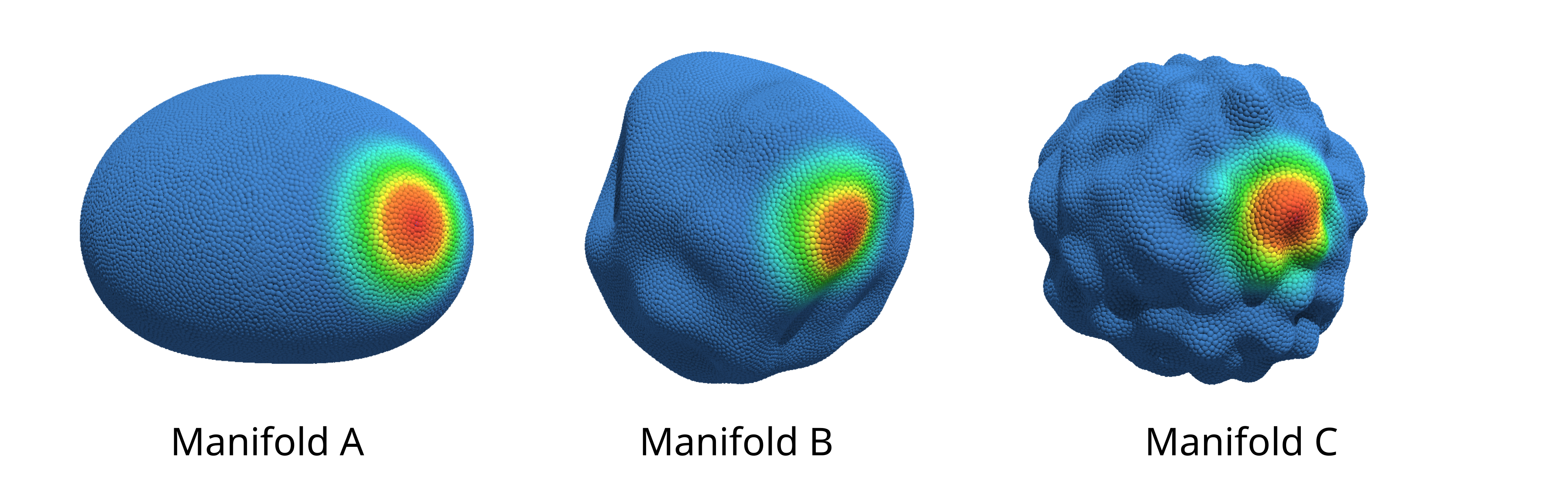}
\caption{\textbf{Manifold Shapes and Kernel Restrictions.} We show
the manifolds used in our comparison studies. We also show how the
Wendland kernel with parameter $k=0$ behaves when restricted to the manifold 
surface to obtain $k_{\mathcal{M}}$.
During training the different kernels are centered at a sampling
of locations $\{x_i\}$ on the manifold providing a collection 
of functions $k_{\sigma}(\cdot,x_i)$ for use during training.
}
\label{fig:manifolds_gaussian}
\end{figure}

We consider the geometric elliptic PDE given by the Laplace-Beltrami equation
\begin{equation}
    \label{eq:laplace_beltrami_pde}
    \begin{cases}
        \Delta_{\text{LB}} \ u(\mb{x}) &= -f(\mb{x}), \\
        \int_{\mathcal{M}} u(\bsy x) \, d\bsy x &= 0,
    \end{cases}
\end{equation} 
where 
\begin{equation}
\label{equ_laplace_beltrami}
\Delta_{LB} = \frac{1}{\sqrt{\abs{g}}} \partial_i \pn{g^{ij} \sqrt{\abs{g}} \partial_j}.
\end{equation}
The $\Delta_{LB}$ generalizes the Laplacian to scalar functions on surfaces
$\mathcal{M}$.  The metric tensor is $g$ with components $g_{ij}$
and the inverse metric tensor $g^{-1}$ with components $g^{ij}$. The 
determinant of the metric is denoted by $|g|$. The partial derivatives 
are taken with respect to the directions of the local coordinate system
at $x$. For more details, see~\cite{Pressley2010,abraham2012manifolds}.  
We perform studies by considering three radial manifolds $A, B, C$ with shapes
having different levels of complexity as shown in
Figure~\ref{fig:manifolds_gaussian}.  We use our operator extension methods
discussed in Section~\ref{sec:operator_extension} and
Figure~\ref{fig:extension_methods}. We use our methods to learn an 
approximation of the solution operator 
$u=\mathcal{S}[f]$ for equation~\ref{eq:laplace_beltrami_pde}.  We train using
kernel evaluations $k(\cdot,x_i)$ to obtain the learned
operator $\mathcal{S}_{k, \theta}$ and use our methods to extend this to
the operator $\mathcal{S}_{k, \theta}^e$ in Section~\ref{sec:operator_extension}
for input functions from the RKHS generated by $k(\cdot,\cdot)$. In our empirical and 
computational studies, we use
throughout the separable factorizations for the kernel operations 
as discussed in Section~\ref{sec:gnp}. 

In more detail, we train our neural operators by consider a training set of
functions of the form $k_\sigma(\cdot,x_i)$ for kernel responses 
with $x_i$ coming from a fixed subset of points $X_N \subset \mathcal{M}$.  
Once training is complete, we extend our neural operators by using our 
kernel approximation approach in Section~\ref{sec:operator_extension}
to obtain $\tilde f$ of $f$
using the regularized optimization problem given 
in equation~\ref{eq:representer} with $\lambda = 10^{-10}$.  
We use the learned neural operator $\mathcal{S}_{k, \theta}$ 
and our kernel methods to obtain the approximate solution 
$\tilde u(x) = \sum_{i=1}^N \alpha_i \mathcal{S}_{k, \theta}(k)(x, x_i)$.  This extends 
the neural operator $\mathcal{S}_{k, \theta}$ trained only on kernel functions $k_{\sigma}(\cdot,x_i)$
to the more general class of functions $f \in \mathcal{H}_k$ in the RKHS $\mathcal{H}_k$ 
generated by $k$. This gives the extended neural operator $\mathcal{S}_{k, \theta}^e$.

To obtain $\mathcal{S}_{k, \theta}$ for each choice of kernel $k(\cdot,\cdot)$
and manifold $\mathcal{M}$, 
we train the neural operator to map $k(\cdot,
x_i) \to \mathcal{S}[k(\cdot, x_i)]$ using $200$ training pairs of functions
$\set{(k(\cdot, x_i), \mathcal{S}[k(\cdot, x_i)]}$ evaluated at
$N_{\text{train}} = 10,000$ points. During training, we use Farthest Point
Sampling (FPS) to subsample $5,000$ out of $N_{\text{train}}$ points and
evaluate the loss in equation~\ref{eq:sobolev_loss}.  All our models consist of
$10$ layers and use for the activation function the Gaussian Error Linear Unit
(GELU)~\cite{Hendrycks2016}. Each kernel network consists of layer widths
$(d_{in}, w/4, w/2, w, d_v^2)$, use latent dimension $d_v=64$ and use width
$w=256$. The input dimension $d_{\text{in}}$ is $4$ for the first $9$ layers
and $3$ in the final layer, as described in Section~\ref{sec:gnp}. All training
is performed for $300$ epochs with an initial learning rate $1e-04$ that is
halved every $50$ epochs.

During operator evaluation, we use approximations based on 
equation~\ref{eq:representer} with $X_N$ using
$N=1,250$, $2,500$, $5,000$, and $10,000$ points.  We compute the relative Sobolev
$\mathcal{H}^1$ error between the true solution $u(x)$ and the pseudo-Green's
function approximant $\tilde u$ at $10,000$ points. We remark that this 
error includes both the function evaluations and the derivatives.
Our ability to train at a
resolution of $5,000$ points and test on $10,000$ points is made possible by
the approximate discretization invariance properties of our neural 
operators. Our neural operators have the property that we can train 
at one level of spatial resolution for the discretizations
and evaluate for new functions at other spatial resolutions. 
For each translation
invariant kernel, we consider two different values for the shape 
parameter $\sigma$ with $k(x, y) =
\Phi \pn{\sigma \pn{x-y}}$. For Gaussian and Mat\'ern ($\nu=\tfrac{1}{2},
\tfrac{3}{2}, \tfrac{5}{2}$) kernels, we use $\sigma=5, 10$. For the Wendland
Kernels of order $k=0, 1, 2$, we use $\sigma=\tfrac{5}{3}, \tfrac{10}{3}$ so
that all kernels have comparable support. 

We use functions for testing that are generated using band-limited spherical
harmonics. For each of the manifolds, we use right-hand side source functions
in equation~\ref{eq:laplace_beltrami_pde} that are generated building on our
previous work on spherical harmonics approaches, geometric PDE solvers, and 
the sphericart
Python package~\cite{sphericart,Atzberger2018b,sigurdsson2016hydrodynamic,rower2023coarse}.
For each function, coefficients were sampled from a normal distribution with
standard deviation inversely proportional to the order squared of the spherical
harmonic up to a specified maximal degree. The maximal degrees used were $3, 6,
8, 10, 12, 15, 18, 22$. We also used $5$ functions for each maximal degree.

\begin{table}[H]
\centering
\resizebox{0.7\linewidth}{!}{
\begin{tabular}{|l|l|l|l|l|l|}
\hline
\rowcolor{black!20!white}
\textbf{Kernel} & \textbf{Shape} $\bsy{\sigma}$ & $\bsy{N=1250}$ &
$\bsy{2500}$ & $\bsy{5000}$ & $\bsy{10000}$ \\ \hline \hline
\rowcolor{black!20!white}
\textbf{Manifold $A$}   & & & & &\\
\hline
\rowcolor{black!5!white}
Gaussian & 10 & 2.16e-01 & 2.18e-01 & 6.69e-01 & \textbf{7.76e+00} \\
\rowcolor{black!5!white}
Gaussian & 5 & \textbf{2.80e+03} & \textbf{1.85e+04} & \textbf{2.80e+04} & \textbf{4.26e+04} \\
\hline
\rowcolor{black!5!white}
Mat\'ern, $\nu=1/2$ & 10 & 1.14e-01 & 1.04e-01 & 1.02e-01 & 1.02e-01 \\
\rowcolor{black!5!white}
Mat\'ern, $\nu=1/2$ & 5 & 1.52e-01 & 1.53e-01 & 1.54e-01 & 1.54e-01 \\
\rowcolor{black!5!white}
Mat\'ern, $\nu=3/2$ & 10 & 6.48e-02 & 6.29e-02 & 6.27e-02 & 6.27e-02 \\
\rowcolor{black!5!white}
Mat\'ern, $\nu=3/2$ & 5 & 1.04e-01 & 1.05e-01 & 1.05e-01 & 1.05e-01 \\
\rowcolor{black!5!white}
Mat\'ern, $\nu=5/2$ & 10 & 7.51e-02 & 7.48e-02 & 7.48e-02 & 7.48e-02 \\
\rowcolor{black!5!white}
Mat\'ern, $\nu=5/2$ & 5 & 8.57e-02 & 8.62e-02 & 8.62e-02 & 8.62e-02 \\
\hline
\rowcolor{black!5!white}
Wendland, $k=0$ & 10/3 & 3.27e-01 & 3.24e-01 & 3.23e-01 & 3.23e-01 \\
\rowcolor{black!5!white}
Wendland, $k=0$ & 5/3 & 3.67e-01 & 3.70e-01 & 3.71e-01 & 3.72e-01 \\
\rowcolor{black!5!white}
Wendland, $k=1$ & 10/3 & 2.28e-01 & 2.27e-01 & 2.27e-01 & 2.27e-01 \\
\rowcolor{black!5!white}
Wendland, $k=1$ & 5/3 & 9.07e-02 & 9.20e-02 & 9.23e-02 & 9.23e-02 \\
\rowcolor{black!5!white}
Wendland, $k=2$ & 10/3 & 5.67e-02 & 5.64e-02 & 5.65e-02 & 5.65e-02 \\
\rowcolor{black!5!white}
Wendland, $k=2$ & 5/3 & 7.62e-02 & 7.72e-02 & 7.72e-02 & 7.72e-02 \\
\hline
\hline
\rowcolor{black!20!white}
\textbf{Manifold $B$}   & & & & &\\
\hline
\rowcolor{black!5!white}
Gaussian & 10 & 8.04e-02 & 8.37e-02 & 9.14e-01 & \textbf{8.38e+00} \\
\rowcolor{black!5!white}
Gaussian & 5 & \textbf{2.04e+01} & \textbf{2.12e+03} & \textbf{3.87e+03} & \textbf{5.80e+03} \\
\hline
\rowcolor{black!5!white}
Mat\'ern, $\nu=1/2$ & 10 & 5.82e-01 & 5.79e-01 & 5.78e-01 & 5.77e-01 \\
\rowcolor{black!5!white}
Mat\'ern, $\nu=1/2$ & 5 & 1.29e-01 & 1.31e-01 & 1.32e-01 & 1.33e-01 \\
\rowcolor{black!5!white}
Mat\'ern, $\nu=3/2$ & 10 & 1.19e-01 & 1.19e-01 & 1.19e-01 & 1.19e-01 \\
\rowcolor{black!5!white}
Mat\'ern, $\nu=3/2$ & 5 & 1.47e-01 & 1.47e-01 & 1.48e-01 & 1.48e-01 \\
\rowcolor{black!5!white}
Mat\'ern, $\nu=5/2$ & 10 & 9.42e-02 & 9.35e-02 & 9.35e-02 & 9.34e-02 \\
\rowcolor{black!5!white}
Mat\'ern, $\nu=5/2$ & 5 & 1.16e-01 & 1.17e-01 & 1.17e-01 & 1.17e-01 \\
\hline
\rowcolor{black!5!white}
Wendland, $k=0$ & 10/3 & 9.75e-02 & 8.84e-02 & 8.68e-02 & 8.65e-02 \\
\rowcolor{black!5!white}
Wendland, $k=0$ & 5/3 & 9.52e-02 & 9.78e-02 & 9.92e-02 & 9.97e-02 \\
\rowcolor{black!5!white}
Wendland, $k=1$ & 10/3 & 2.66e-01 & 2.65e-01 & 2.65e-01 & 2.65e-01 \\
\rowcolor{black!5!white}
Wendland, $k=1$ & 5/3 & 9.68e-02 & 9.79e-02 & 9.81e-02 & 9.82e-02 \\
\rowcolor{black!5!white}
Wendland, $k=2$ & 10/3 & 8.29e-02 & 8.24e-02 & 8.25e-02 & 8.25e-02 \\
\rowcolor{black!5!white}
Wendland, $k=2$ & 5/3 & 9.47e-02 & 9.62e-02 & 9.63e-02 & 9.63e-02 \\
\hline
\hline
\rowcolor{black!20!white}
\textbf{Manifold $C$}   & & & & &\\
\hline
\rowcolor{black!5!white}
Gaussian & 10 & 1.54e-01 & 1.55e-01 & 1.92e-01 & \textbf{3.44e+00} \\
\rowcolor{black!5!white}
Gaussian & 5 & \textbf{3.16e+00} & \textbf{1.74e+02} & \textbf{1.86e+03} & \textbf{2.88e+03} \\
\hline
\rowcolor{black!5!white}
Mat\'ern, $\nu=1/2$ & 10 & 1.45e-01 & 1.35e-01 & 1.32e-01 & 1.32e-01 \\
\rowcolor{black!5!white}
Mat\'ern, $\nu=1/2$ & 5 & 1.41e-01 & 1.42e-01 & 1.43e-01 & 1.43e-01 \\
\rowcolor{black!5!white}
Mat\'ern, $\nu=3/2$ & 10 & 1.18e-01 & 1.17e-01 & 1.17e-01 & 1.17e-01 \\
\rowcolor{black!5!white}
Mat\'ern, $\nu=3/2$ & 5 & 1.40e-01 & 1.41e-01 & 1.41e-01 & 1.41e-01 \\
\rowcolor{black!5!white}
Mat\'ern, $\nu=5/2$ & 10 & 1.12e-01 & 1.12e-01 & 1.12e-01 & 1.12e-01 \\
\rowcolor{black!5!white}
Mat\'ern, $\nu=5/2$ & 5 & 1.62e-01 & 1.64e-01 & 1.65e-01 & 1.66e-01 \\
\hline
\rowcolor{black!5!white}
Wendland, $k=0$ & 10/3 & 1.66e-01 & 1.59e-01 & 1.58e-01 & 1.58e-01 \\
\rowcolor{black!5!white}
Wendland, $k=0$ & 5/3 & 1.42e-01 & 1.44e-01 & 1.46e-01 & 1.46e-01 \\
\rowcolor{black!5!white}
Wendland, $k=1$ & 10/3 & 1.49e-01 & 1.47e-01 & 1.47e-01 & 1.47e-01 \\
\rowcolor{black!5!white}
Wendland, $k=1$ & 5/3 & 1.42e-01 & 1.47e-01 & 1.49e-01 & 1.49e-01 \\
\rowcolor{black!5!white}
Wendland, $k=2$ & 10/3 & 1.14e-01 & 1.13e-01 & 1.13e-01 & 1.13e-01 \\
\rowcolor{black!5!white}
Wendland, $k=2$ & 5/3 & 1.32e-01 & 1.42e-01 & 1.46e-01 & 1.48e-01 \\
\hline
\end{tabular}
}
\caption{\textbf{Relative }$\bsy{\mathcal{H}^1}$\textbf{-norm Accuracy.} We show 
results for the accuracy of the extended neural operators when used to obtain
solutions of equation~\ref{eq:laplace_beltrami_pde}.  We show in bold all
relative errors above $10^0$.  The Gaussian kernels are found to perform 
poorly as $N$ increases.  The extension methods using the Mat\'ern and Wendland
kernels are found to perform well as $N$ increases.
}
\label{tab:gnp_h1_error}

\end{table}

We study the relative accuracy in the $\mathcal{H}^1$-norm
for our extended neural operator predictions of solution functions 
$\tilde u$ for the geometric PDE in equation~\ref{eq:pseudo_green_solution}.
We report values for each manifold-kernel pair averaged over all test functions
in Table~\ref{tab:gnp_h1_error}.  We find the Gaussian kernels perform significantly
worse than the Mat\'ern and Wendland kernels.
As part of our empirical studies when solving equation~\ref{eq:representer}
with $N$ kernel points, we track the $\ell^1$-norm of the coefficients
$\set{\alpha_i}$ in equation~\ref{eq:rkhs_approximant}.  
The $\ell^1$-norm contributes to the error estimate through the term $C_2$ in 
Theorems~\ref{thm:op_extend} in equation~\ref{thm:op_extend_manifold}.  
We report these
results in Table~\ref{tab:gnp_l1_coefficient_norm}.

We find in the case of the Mat\'ern and Wendland
kernels that the $\norm{\alpha}_{\ell^1}$ is on the order of $10^3$ or less.
These kernels maintain
a similar order of magnitude as $N$ increases. In contrast for Gaussian kernels,
the $\norm{\alpha}_{\ell^1}$ increases dramatically as $N$ increases. This 
appears to be related to the well-known ill-conditioning of the 
Gram matrix $K$ for Gaussian kernels as the 
separation radius of the point sampling decreases. 
Additional discussions of the general theory of stability of kernel methods
can be found in~\cite{Wendland2004}.
We present condition numbers
for all Gram matrices used our studies in Appendix~\ref{sec:appendix_a}. 

Our results show there is a significant impact of the ill-conditioning of the
Gaussian kernel which greatly amplifies the error in the learned pseudo-Green's
functions.  This degrades for Gaussians as the amount of overlap between kernel
functions increases.  We see this is most pronounced when the kernel has a
larger support as the $\sigma$ value becomes smaller, 
see Table~\ref{tab:gnp_h1_error}.

We find when solving equation~\ref{eq:representer} the Mat\'ern kernels
exhibit significantly better conditioning in $K$ than Gaussian kernels.
This results in  
better approximation
when super-imposing the pseudo-Green's functions.  Further, we see in almost
all cases that for the Mat\'ern kernels there is more consistency exhibited
when varying the number of kernel points $N$ when solving
equation~\ref{eq:representer}.  For the manifold $A$, all errors vary between
$6\%$ and $15\%$ and exhibit better approximation when taking $\sigma=10$.  The
lowest error of $6.3\%$ occurs when $\nu=3/2$.  The manifold $B$ also generally
shows better results when $\sigma=10$ with errors varying between $9\%$ and
$15\%$ with the exception of $\nu=1/2$. When $\nu=1/2$ and $\sigma=10$, 
we find the larger $\mathcal{H}^1$ errors are a result of inaccurate gradient
approximations during training, see discussions in 
Appendix~\ref{sec:appendix_a}. The increase in errors for the manifold $C$ is
somewhat expected since it has a more complicated shape and geometric variations. 
However, despite this we still see
our methods are able to achieve consistent results 
between $11\%$ and $17\%$ error. We also remark that the
trends across all manifolds for the Mat\'ern kernel show better results for
$\sigma=10$ and $\nu > 1/2$.

Wendland kernels differ from the Mat\'ern and Gaussian kernels by having compact
support vanishing outside a ball of radius $r$. In practice, this results in 
efficiencies in evaluation and even 
better conditioning in
the Gram matrix $K$ especially when compared to the Gaussian
kernels~\cite{Wendland2004}. As a consequence of Wendland kernel's compact support, 
our neural operator training methods are variable due to the lack of signal at 
locations outside of the kernel support. We see this when evaluating on the 
manifolds $A$, $B$, particularly when $k < 2$. Despite this, we still see 
strong performance for order $k=2$ across all three manifolds, which 
attains for manifolds $A, B, C$ errors of $5.7\%$, $8.3\%$, and $11.4\%$.

Our results show the trade-offs that occur between the choice of kernel type,
hyper-parameters, and geometry. Our studies further show choosing a kernel with smaller
support (larger $\sigma$) 
enables better conditioning of the Gram Matrix which in turn reduces the 
upper bound $C_2$ in Theorem~\ref{thm:op_extend} and Theorem~\ref{thm:op_extend_manifold}.
We also find if $\sigma$ is too large this can lead to difficulties in training for the GNPs,
as would occur also for other neural operators and kernel methods. Further, we find  
choosing smoother kernels with good conditioning such as the Mat\'ern ($\nu=5/2)$ or 
Wendland ($k=2$) kernels achieves a good balance yielding the best results across 
all three test manifolds $A, B, C$, see Table~\ref{tab:gnp_h1_error}. 
These results show the importance
of using alternatives to Gaussian kernels to obtain good performance.

\begin{table}[htbp]
\centering
\resizebox{0.85\linewidth}{!}{
\begin{tabular}{|l|l|l|l|l|l|}
\hline
\rowcolor{black!20!white}
 \textbf{Kernel} & \textbf{Shape} $\bsy \sigma$ & $\bsy{N=1250}$ & 
 $\bsy{2500}$ & $\bsy{5000}$ & $\bsy{10000}$ \\ \hline
\rowcolor{black!5!white}
Gaussian & 10 & 1.34e+03 & 6.65e+03 & \textbf{6.32e+05} & \textbf{2.05e+07} \\
\rowcolor{black!5!white}
Gaussian & 5 & \textbf{2.05e+08} & \textbf{3.08e+09} & \textbf{6.84e+09} & \textbf{1.13e+10} \\
\hline
\rowcolor{black!5!white}
Mat\'ern, $\nu=1/2$ & 10 & 1.28e+03 & 1.49e+03 & 1.59e+03 & 1.63e+03 \\
\rowcolor{black!5!white}
Mat\'ern, $\nu=1/2$ & 5 & 1.53e+03 & 1.79e+03 & 1.91e+03 & 1.96e+03 \\
\rowcolor{black!5!white}
Mat\'ern, $\nu=3/2$ & 10 & 1.17e+03 & 1.26e+03 & 1.30e+03 & 1.32e+03 \\
\rowcolor{black!5!white}
Mat\'ern, $\nu=3/2$ & 5 & 2.26e+03 & 2.57e+03 & 2.77e+03 & 2.89e+03 \\
\rowcolor{black!5!white}
Mat\'ern, $\nu=5/2$ & 10 & 1.15e+03 & 1.25e+03 & 1.34e+03 & 1.46e+03 \\
\rowcolor{black!5!white}
Mat\'ern, $\nu=5/2$ & 5 & 3.64e+03 & 4.87e+03 & 6.58e+03 & 8.43e+03 \\
\hline
\rowcolor{black!5!white}
Wendland, $k=0$ & 10/3 & 1.21e+03 & 1.36e+03 & 1.43e+03 & 1.47e+03 \\
\rowcolor{black!5!white}
Wendland, $k=0$ & 5/3 & 1.96e+03 & 2.31e+03 & 2.47e+03 & 2.55e+03 \\
\rowcolor{black!5!white}
Wendland, $k=1$ & 10/3 & 1.22e+03 & 1.31e+03 & 1.38e+03 & 1.43e+03 \\
\rowcolor{black!5!white}
Wendland, $k=1$ & 5/3 & 3.38e+03 & 3.98e+03 & 4.41e+03 & 4.67e+03 \\
\rowcolor{black!5!white}
Wendland, $k=2$ & 10/3 & 4.64e+02 & 4.95e+02 & 5.52e+02 & 6.25e+02 \\
\rowcolor{black!5!white}
Wendland, $k=2$ & 5/3 & 1.26e+03 & 1.94e+03 & 2.88e+03 & 3.93e+03 \\
\hline
\end{tabular}
}
\caption{\textbf{Coefficient $\bsy{\ell^1}$-norms and $\bsy{C_2}$.}  We show  
the $\|\alpha\|_{\ell^1}$ of the coefficients 
$\{\alpha_i\}$  in equation~\ref{eq:rkhs_approximant} when solving 
equation~\ref{eq:representer}. Results are averaged over all functions and manifolds.
This gives an indication of the size of $C_2 = \|\alpha\|_{\ell^1}$ in the 
accuracy bounds in Theorem~\ref{thm:op_extend_manifold} in equation~\ref{eq:thm_ext_acc_M}.
We highlight in bold all values larger than $10^4$.  As the
number of sample points increases, we see 
the Gaussian kernels have especially large values 
exceeding those of our other kernels by a few orders of
magnitude.  Relative to the other kernels, this indicates 
using Gaussians in the extension methods could be 
inefficient and sensitive to the accuracy of the trained 
neural operators, see 
equation~\ref{eq:thm_ext_acc_M}. 
}
\label{tab:gnp_l1_coefficient_norm}
\end{table}

\section*{Conclusions}

We have shown how learned neural operators can be refined for enhanced
robustness and improved generalization to out-of-distribution input functions.
By leveraging kernel approximation techniques, we achieve well-controlled
responses across a broad range of functions beyond those encountered in the
training data. We show that selecting kernels corresponding to
Reproducing Kernel Hilbert Spaces (RKHSs) associated with Sobolev
spaces enables learned operators and extensions to capture mappings 
of both functions and their
derivatives. Beyond our methodological contributions, we establish a
theoretical framework to characterize the approximation accuracy and 
identify key error factors. This includes the critical roles 
played by the point sampling, kernel type, and kernel bandwidth. 
The practical utility of the methods was validated through applications 
for solving elliptic PDEs, operators on manifolds having point-cloud 
representations, and handling geometric contributions.  The methods 
offer systematic approaches for empirically controlling 
and extending neural operators for diverse learning tasks.

\section*{Acknowledgements}

Authors research supported by grant NSF Grant DMS-2306101 and NSF Grant NSF-DMS-2306345.  Authors also would
like to acknowledge computational resources and administrative support at
the UCSB Center for Scientific Computing (CSC) with grants
NSF-CNS-1725797, MRSEC: NSF-DMR-2308708, Pod-GPUs: OAC-1925717, and support from
the California NanoSystems Institute (CNSI) at UCSB.  P.J.A. also would
like to acknowledge a hardware grant from Nvidia.

\newpage

\printbibliography

\appendix

\section{Additional Results: Role of Kernel Choices and Sources of Error}
\label{sec:appendix_a}

We present additional results on the condition number of the Gram matrix for
the various kernels and resolutions prior to adding any regularization when
solving equation~\ref{eq:representer}.  These results are shown in in
Table~~\ref{tab:condition_number}.  Our results especially highlight well-known
ill-conditioning phenomenon that can arise when using Gaussian kernels.  Our
results further show how Mat\'ern and Wendland kernels can be used to obtain
much better conditioning for systems. We also show how the condition number
grows with the number of points $N$ and the impact of the smoothness parameters
$\nu$ and $k$, see Table~~\ref{tab:condition_number}.

We also show results for the relative $L^2$-errors for $\tilde u$ and $\nabla
\tilde u$ in Tables~\ref{tab:gnp_l2_error} and~\ref{tab:gnp_l2_grad_error}.  We
find in our studies that nearly all of the relative errors in
Table~\ref{tab:gnp_l2_error} are below 10\%. This highlights that the majority
of the contributions to the $\mathcal{H}^1$ error in
Table~\ref{tab:gnp_h1_error} arise from errors in approximating the gradient
$\nabla u$.  For the Gaussian kernel case, we see again signs of the
ill-conditioning arising even in the error for the functions $u$.  This becomes
increasingly an issue for Gaussian kernels as $N$ becomes large when
approximating the the source function $f$.  These errors then propagate through
the other calculations in the Gaussian kernel case, see 
Tables~\ref{tab:gnp_l2_error} and~\ref{tab:gnp_l2_grad_error}. 

In summary, our studies and results further highlight the importance of the choice of the
radial kernel functions used in our extension methods and for training neural
operators more generally. When using alternative kernels, such as the Mat\'ern
and Wendland kernels we find much better performance both in capturing the
functions $u$ and even in some cases their derivatives.  
\begin{table}[htb]
\centering
\resizebox{0.86\linewidth}{!}{
\begin{tabular}{|l|l|l|l|l|l|}
\hline
\rowcolor{black!20!white}
\textbf{Kernel} & \textbf{Shape} $\bsy{\sigma}$ & $\bsy{N=1250}$ & 
$\bsy{2500}$ & $\bsy{5000}$ & $\bsy{10000}$ \\ \hline
\rowcolor{black!5!white}
Gaussian & 10 & 1.07e+03 & 7.45e+05 & \textbf{3.99e+11} & \textbf{6.36e+19} \\
\rowcolor{black!5!white}
Gaussian & 5 & \textbf{1.21e+10} & \textbf{5.36e+18} & \textbf{1.19e+20} & \textbf{1.23e+21} \\
\hline
\rowcolor{black!5!white}
Mat\'ern, $\nu=1/2$ & 10 & 1.61e+01 & 4.33e+01 & 1.34e+02 & 3.59e+02 \\
\rowcolor{black!5!white}
Mat\'ern, $\nu=1/2$ & 5 & 1.19e+02 & 3.36e+02 & 1.07e+03 & 2.84e+03 \\
\rowcolor{black!5!white}
Mat\'ern, $\nu=3/2$ & 10 & 3.72e+01 & 1.77e+02 & 1.25e+03 & 6.01e+03 \\
\rowcolor{black!5!white}
Mat\'ern, $\nu=3/2$ & 5 & 8.81e+02 & 4.82e+03 & 3.73e+04 & 1.83e+05 \\
\rowcolor{black!5!white}
Mat\'ern, $\nu=5/2$ & 10 & 6.43e+01 & 4.83e+02 & 7.11e+03 & 6.14e+04 \\
\rowcolor{black!5!white}
Mat\'ern, $\nu=5/2$ & 5 & 4.12e+03 & 4.22e+04 & 7.61e+05 & \textbf{7.06e+06} \\
\hline
\rowcolor{black!5!white}
Wendland, $k=0$ & 10/3 & 1.76e+01 & 4.90e+01 & 1.55e+02 & 4.16e+02 \\
\rowcolor{black!5!white}
Wendland, $k=0$ & 5/3 & 1.33e+02 & 3.77e+02 & 1.19e+03 & 3.21e+03 \\
\rowcolor{black!5!white}
Wendland, $k=1$ & 10/3 & 4.27e+01 & 2.33e+02 & 1.82e+03 & 9.03e+03 \\
\rowcolor{black!5!white}
Wendland, $k=1$ & 5/3 & 1.21e+03 & 6.93e+03 & 5.51e+04 & 2.74e+05 \\
\rowcolor{black!5!white}
Wendland, $k=2$ & 10/3 & 3.27e+01 & 3.36e+02 & 6.29e+03 & 5.95e+04 \\
\rowcolor{black!5!white}
Wendland, $k=2$ & 5/3 & 3.32e+03 & 3.81e+04 & 7.36e+05 & \textbf{7.05e+06} \\
\hline
\end{tabular}
}
\caption{\textbf{Condition Numbers for the Gram Matrix.} We 
show results for the Gram Matrix $K_{ij} = k(x_i,x_j)$ for 
the Gaussian, Mat\'ern, and Wendland kernels.  We highlight in 
bold conditions numbers larger than $10^6$. 
We find as $N$ increases 
the Gaussian kernels produce particularly ill-conditioned 
gram matrices which are orders of magnitude larger than the other kernels.
The condition numbers are seen to grow much more slowly for the 
Mat\'ern and Wendland kernels.  The larger condition 
numbers can be seen to correspond to the less accurate cases
in Table~\ref{tab:gnp_l1_coefficient_norm}.
These results further highlight 
importance of 
the choice of kernel when using the extension methods.
}
\label{tab:condition_number}
\end{table}

\newpage
\begin{table}[p]
\centering
\resizebox{0.7\linewidth}{!}{
\begin{tabular}{|l|l|l|l|l|l|}
\hline
\rowcolor{black!20!white}
 \textbf{Kernel} & \textbf{Shape} $\bsy{\sigma}$ & $\bsy{N = 1250}$ & $\bsy{2500}$ 
& $\bsy{5000}$ & $\bsy{10000}$ \\ \hline \hline
\rowcolor{black!20!white}
\textbf{Manifold $A$}  & & & & & \\
\hline
\rowcolor{black!5!white}
Gaussian & 10 & 1.10e-01 & 1.11e-01 & 5.02e-01 & \textbf{5.13e+00} \\
\rowcolor{black!5!white}
Gaussian & 5 & \textbf{1.79e+03} & \textbf{1.25e+04} & \textbf{1.94e+04} & \textbf{2.99e+04} \\
\hline
\rowcolor{black!5!white}
Mat\'ern, $\nu=1/2$ & 10 & 6.09e-02 & 4.50e-02 & 4.26e-02 & 4.24e-02 \\
\rowcolor{black!5!white}
Mat\'ern, $\nu=1/2$ & 5 & 5.37e-02 & 5.45e-02 & 5.52e-02 & 5.55e-02 \\
\rowcolor{black!5!white}
Mat\'ern, $\nu=3/2$ & 10 & 2.93e-02 & 2.81e-02 & 2.83e-02 & 2.83e-02 \\
\rowcolor{black!5!white}
Mat\'ern, $\nu=3/2$ & 5 & 3.85e-02 & 3.90e-02 & 3.92e-02 & 3.92e-02 \\
\rowcolor{black!5!white}
Mat\'ern, $\nu=5/2$ & 10 & 3.13e-02 & 3.12e-02 & 3.13e-02 & 3.13e-02 \\
\rowcolor{black!5!white}
Mat\'ern, $\nu=5/2$ & 5 & 4.73e-02 & 4.77e-02 & 4.77e-02 & 4.77e-02 \\
\hline
\rowcolor{black!5!white}
Wendland, $k=0$ & 10/3 & 2.12e-01 & 2.08e-01 & 2.07e-01 & 2.07e-01 \\
\rowcolor{black!5!white}
Wendland, $k=0$ & 5/3 & 6.11e-02 & 6.43e-02 & 6.58e-02 & 6.64e-02 \\
\rowcolor{black!5!white}
Wendland, $k=1$ & 10/3 & 9.18e-02 & 9.17e-02 & 9.18e-02 & 9.18e-02 \\
\rowcolor{black!5!white}
Wendland, $k=1$ & 5/3 & 4.94e-02 & 5.08e-02 & 5.11e-02 & 5.11e-02 \\
\rowcolor{black!5!white}
Wendland, $k=2$ & 10/3 & 2.81e-02 & 3.07e-02 & 3.11e-02 & 3.11e-02 \\
\rowcolor{black!5!white}
Wendland, $k=2$ & 5/3 & 3.30e-02 & 3.33e-02 & 3.33e-02 & 3.33e-02 \\
\hline
\hline
\rowcolor{black!20!white}
\textbf{Manifold $B$}   & & & & &\\
\hline
\rowcolor{black!5!white}
Gaussian & 10 & 2.77e-02 & 2.92e-02 & 3.75e-01 & \textbf{3.88e+00} \\
\rowcolor{black!5!white}
Gaussian & 5 & \textbf{1.15e+01} & \textbf{1.27e+03} & \textbf{2.25e+03} & \textbf{3.44e+03} \\
\hline
\rowcolor{black!5!white}
Mat\'ern, $\nu=1/2$ & 10 & 7.51e-02 & 6.11e-02 & 5.82e-02 & 5.77e-02 \\
\rowcolor{black!5!white}
Mat\'ern, $\nu=1/2$ & 5 & 5.57e-02 & 5.78e-02 & 5.85e-02 & 5.88e-02 \\
\rowcolor{black!5!white}
Mat\'ern, $\nu=3/2$ & 10 & 3.34e-02 & 3.16e-02 & 3.15e-02 & 3.15e-02 \\
\rowcolor{black!5!white}
Mat\'ern, $\nu=3/2$ & 5 & 5.64e-02 & 5.75e-02 & 5.76e-02 & 5.76e-02 \\
\rowcolor{black!5!white}
Mat\'ern, $\nu=5/2$ & 10 & 2.79e-02 & 2.77e-02 & 2.78e-02 & 2.78e-02 \\
\rowcolor{black!5!white}
Mat\'ern, $\nu=5/2$ & 5 & 6.06e-02 & 6.16e-02 & 6.17e-02 & 6.18e-02 \\
\rowcolor{black!5!white}
\hline
\rowcolor{black!5!white}
Wendland, $k=0$ & 10/3 & 4.95e-02 & 3.28e-02 & 2.96e-02 & 2.92e-02 \\
\rowcolor{black!5!white}
Wendland, $k=0$ & 5/3 & 5.09e-02 & 5.41e-02 & 5.55e-02 & 5.60e-02 \\
\rowcolor{black!5!white}
Wendland, $k=1$ & 10/3 & 6.35e-02 & 6.55e-02 & 6.60e-02 & 6.61e-02 \\
\rowcolor{black!5!white}
Wendland, $k=1$ & 5/3 & 4.87e-02 & 4.99e-02 & 5.01e-02 & 5.02e-02 \\
\rowcolor{black!5!white}
Wendland, $k=2$ & 10/3 & 2.71e-02 & 2.78e-02 & 2.80e-02 & 2.81e-02 \\
\rowcolor{black!5!white}
Wendland, $k=2$ & 5/3 & 3.88e-02 & 3.97e-02 & 4.00e-02 & 4.00e-02 \\
\hline
\hline
\rowcolor{black!20!white}
\textbf{Manifold $C$}   & & & & &\\
\hline
\rowcolor{black!5!white}
Gaussian & 10 & 3.62e-02 & 4.33e-02 & 1.05e-01 & \textbf{2.97e+00} \\
\rowcolor{black!5!white}
Gaussian & 5 & \textbf{1.54e+00} & \textbf{7.43e+01} & \textbf{7.33e+02} & \textbf{1.16e+03} \\
\hline
\rowcolor{black!5!white}
Mat\'ern, $\nu=1/2$ & 10 & 5.86e-02 & 4.03e-02 & 3.79e-02 & 3.79e-02 \\
\rowcolor{black!5!white}
Mat\'ern, $\nu=1/2$ & 5 & 5.93e-02 & 6.38e-02 & 6.60e-02 & 6.69e-02 \\
\rowcolor{black!5!white}
Mat\'ern, $\nu=3/2$ & 10 & 3.59e-02 & 3.46e-02 & 3.51e-02 & 3.53e-02 \\
\rowcolor{black!5!white}
Mat\'ern, $\nu=3/2$ & 5 & 6.81e-02 & 7.12e-02 & 7.22e-02 & 7.25e-02 \\
\rowcolor{black!5!white}
Mat\'ern, $\nu=5/2$ & 10 & 3.54e-02 & 3.38e-02 & 3.37e-02 & 3.37e-02 \\
\rowcolor{black!5!white}
Mat\'ern, $\nu=5/2$ & 5 & 8.02e-02 & 8.72e-02 & 8.98e-02 & 9.07e-02 \\
\hline
\rowcolor{black!5!white}
Wendland, $k=0$ & 10/3 & 6.07e-02 & 4.70e-02 & 4.60e-02 & 4.64e-02 \\
\rowcolor{black!5!white}
Wendland, $k=0$ & 5/3 & 6.43e-02 & 7.01e-02 & 7.31e-02 & 7.42e-02 \\
\rowcolor{black!5!white}
Wendland, $k=1$ & 10/3 & 4.74e-02 & 4.28e-02 & 4.23e-02 & 4.21e-02 \\
\rowcolor{black!5!white}
Wendland, $k=1$ & 5/3 & 6.73e-02 & 7.71e-02 & 8.07e-02 & 8.18e-02 \\
\rowcolor{black!5!white}
Wendland, $k=2$ & 10/3 & 3.06e-02 & 2.61e-02 & 2.58e-02 & 2.57e-02 \\
\rowcolor{black!5!white}
Wendland, $k=2$ & 5/3 & 6.25e-02 & 8.14e-02 & 8.82e-02 & 9.04e-02 \\
\hline
\end{tabular}
}
\caption{\textbf{$\bsy{L^2}$-norm Relative Errors for 
$\bsy{\tilde{u}}$}. We show the accuracy of approximating 
$\bsy {\tilde u}$ when solving equation~\ref{eq:representer}.  We show in bold
errors larger than $1.0$.  The accuracy of these results indicates that the main source
of error in the Sobolev-norm in Table~\ref{tab:gnp_h1_error} likely arises from 
approximation of the gradients. We report these results in Table~\ref{tab:gnp_l2_grad_error}. 
}
\label{tab:gnp_l2_error}
\end{table}

\newpage
\begin{table}[p]
\centering
\resizebox{0.7\linewidth}{!}{
\begin{tabular}{|l|l|l|l|l|l|}
\hline
\rowcolor{black!20!white}
 \textbf{Kernel} & \textbf{Shape} $\bsy{\sigma}$ & $\bsy{N=1250}$ & $\bsy{2500}$ 
& $\bsy{5000}$ & $\bsy{10000}$ \\ \hline \hline
\rowcolor{black!20!white}
\textbf{Manifold $A$}   & & & & &\\
\hline
\rowcolor{black!5!white}
Gaussian & 10 & 2.42e-01 & 2.44e-01 & 7.18e-01 & \textbf{8.43e+00} \\
\rowcolor{black!5!white}
Gaussian & 5 & \textbf{3.04e+03} & \textbf{2.01e+04} & \textbf{3.03e+04} & \textbf{4.61e+04} \\
\hline
\rowcolor{black!5!white}
Mat\'ern, $\nu=1/2$ & 10 & 1.27e-01 & 1.16e-01 & 1.15e-01 & 1.15e-01 \\
\rowcolor{black!5!white}
Mat\'ern, $\nu=1/2$ & 5 & 1.70e-01 & 1.72e-01 & 1.73e-01 & 1.73e-01 \\
\rowcolor{black!5!white}
Mat\'ern, $\nu=3/2$ & 10 & 7.25e-02 & 7.05e-02 & 7.03e-02 & 7.03e-02 \\
\rowcolor{black!5!white}
Mat\'ern, $\nu=3/2$ & 5 & 1.17e-01 & 1.17e-01 & 1.17e-01 & 1.17e-01 \\
\rowcolor{black!5!white}
Mat\'ern, $\nu=5/2$ & 10 & 8.43e-02 & 8.40e-02 & 8.39e-02 & 8.39e-02 \\
\rowcolor{black!5!white}
Mat\'ern, $\nu=5/2$ & 5 & 9.45e-02 & 9.51e-02 & 9.51e-02 & 9.51e-02 \\
\hline
\rowcolor{black!5!white}
Wendland, $k=0$ & 10/3 & 3.59e-01 & 3.56e-01 & 3.55e-01 & 3.55e-01 \\
\rowcolor{black!5!white}
Wendland, $k=0$ & 5/3 & 4.20e-01 & 4.23e-01 & 4.24e-01 & 4.25e-01 \\
\rowcolor{black!5!white}
Wendland, $k=1$ & 10/3 & 2.59e-01 & 2.57e-01 & 2.57e-01 & 2.57e-01 \\
\rowcolor{black!5!white}
Wendland, $k=1$ & 5/3 & 1.00e-01 & 1.02e-01 & 1.02e-01 & 1.02e-01 \\
\rowcolor{black!5!white}
Wendland, $k=2$ & 10/3 & 6.34e-02 & 6.27e-02 & 6.27e-02 & 6.27e-02 \\
\rowcolor{black!5!white}
Wendland, $k=2$ & 5/3 & 8.50e-02 & 8.62e-02 & 8.62e-02 & 8.62e-02 \\
\hline
\hline
\rowcolor{black!20!white}
\textbf{Manifold $B$}   & & & & &\\
\hline
\rowcolor{black!5!white}
Gaussian & 10 & 9.18e-02 & 9.54e-02 & \textbf{1.02e+00} & \textbf{9.32e+00} \\
\rowcolor{black!5!white}
Gaussian & 5 & \textbf{2.23e+01} & \textbf{2.31e+03} & \textbf{4.22e+03} & \textbf{6.32e+03} \\
\hline
\rowcolor{black!5!white}
Mat\'ern, $\nu=1/2$ & 10 & 6.76e-01 & 6.73e-01 & 6.72e-01 & 6.71e-01 \\
\rowcolor{black!5!white}
Mat\'ern, $\nu=1/2$ & 5 & 1.45e-01 & 1.47e-01 & 1.48e-01 & 1.49e-01 \\
\rowcolor{black!5!white}
Mat\'ern, $\nu=3/2$ & 10 & 1.37e-01 & 1.36e-01 & 1.36e-01 & 1.36e-01 \\
\rowcolor{black!5!white}
Mat\'ern, $\nu=3/2$ & 5 & 1.66e-01 & 1.66e-01 & 1.66e-01 & 1.67e-01 \\
\rowcolor{black!5!white}
Mat\'ern, $\nu=5/2$ & 10 & 1.08e-01 & 1.07e-01 & 1.07e-01 & 1.07e-01 \\
\rowcolor{black!5!white}
Mat\'ern, $\nu=5/2$ & 5 & 1.29e-01 & 1.30e-01 & 1.30e-01 & 1.30e-01 \\
\rowcolor{black!5!white}
\hline
\rowcolor{black!5!white}
Wendland, $k=0$ & 10/3 & 1.09e-01 & 1.01e-01 & 9.91e-02 & 9.88e-02 \\
\rowcolor{black!5!white}
Wendland, $k=0$ & 5/3 & 1.06e-01 & 1.08e-01 & 1.10e-01 & 1.10e-01 \\
\rowcolor{black!5!white}
Wendland, $k=1$ & 10/3 & 3.07e-01 & 3.06e-01 & 3.05e-01 & 3.05e-01 \\
\rowcolor{black!5!white}
Wendland, $k=1$ & 5/3 & 1.08e-01 & 1.09e-01 & 1.09e-01 & 1.09e-01 \\
\rowcolor{black!5!white}
Wendland, $k=2$ & 10/3 & 9.49e-02 & 9.43e-02 & 9.43e-02 & 9.43e-02 \\
\rowcolor{black!5!white}
Wendland, $k=2$ & 5/3 & 1.07e-01 & 1.09e-01 & 1.09e-01 & 1.09e-01 \\
\hline
\hline
\rowcolor{black!20!white}
\textbf{Manifold $C$}  & & & & & \\
\hline
\rowcolor{black!5!white}
Gaussian & 10 & 1.79e-01 & 1.80e-01 & 2.15e-01 & \textbf{3.61e+00} \\
\rowcolor{black!5!white}
Gaussian & 5 & \textbf{3.60e+00} & \textbf{2.00e+02} & \textbf{2.15e+03} & \textbf{3.33e+03} \\
\hline
\rowcolor{black!5!white}
Mat\'ern, $\nu=1/2$ & 10 & 1.66e-01 & 1.56e-01 & 1.53e-01 & 1.53e-01 \\
\rowcolor{black!5!white}
Mat\'ern, $\nu=1/2$ & 5 & 1.61e-01 & 1.61e-01 & 1.61e-01 & 1.62e-01 \\
\rowcolor{black!5!white}
Mat\'ern, $\nu=3/2$ & 10 & 1.36e-01 & 1.35e-01 & 1.35e-01 & 1.35e-01 \\
\rowcolor{black!5!white}
Mat\'ern, $\nu=3/2$ & 5 & 1.58e-01 & 1.59e-01 & 1.59e-01 & 1.59e-01 \\
\rowcolor{black!5!white}
Mat\'ern, $\nu=5/2$ & 10 & 1.29e-01 & 1.29e-01 & 1.29e-01 & 1.29e-01 \\
\rowcolor{black!5!white}
Mat\'ern, $\nu=5/2$ & 5 & 1.82e-01 & 1.84e-01 & 1.85e-01 & 1.85e-01 \\
\hline
\rowcolor{black!5!white}
Wendland, $k=0$ & 10/3 & 1.90e-01 & 1.84e-01 & 1.83e-01 & 1.83e-01 \\
\rowcolor{black!5!white}
Wendland, $k=0$ & 5/3 & 1.61e-01 & 1.63e-01 & 1.64e-01 & 1.64e-01 \\
\rowcolor{black!5!white}
Wendland, $k=1$ & 10/3 & 1.72e-01 & 1.70e-01 & 1.70e-01 & 1.70e-01 \\
\rowcolor{black!5!white}
Wendland, $k=1$ & 5/3 & 1.61e-01 & 1.65e-01 & 1.67e-01 & 1.67e-01 \\
\rowcolor{black!5!white}
Wendland, $k=2$ & 10/3 & 1.33e-01 & 1.32e-01 & 1.32e-01 & 1.32e-01 \\
\rowcolor{black!5!white}
Wendland, $k=2$ & 5/3 & 1.50e-01 & 1.59e-01 & 1.63e-01 & 1.64e-01 \\
\hline
\end{tabular}
}
\caption{\textbf{$\bsy{L^2}$-norm Relative Errors for 
$\bsy{\nabla \tilde u}$.}
We show the accuracy of approximating 
$\nabla \tilde u$ when solving equation~\ref{eq:representer}.  We show in bold
errors larger than $1.0$.  The accuracy of these results indicates that the 
gradients are the dominant source of approximation error in the 
Sobolev-norm in Table~\ref{tab:gnp_h1_error}.}
\label{tab:gnp_l2_grad_error}
\end{table}

\end{document}